\documentclass[journal]{IEEEtran}
\usepackage{amsfonts}
\usepackage{amsmath}
\usepackage{amsthm}
\usepackage[dvipsnames]{xcolor}
\usepackage[colorlinks=True,
            linkcolor=cyan,
            anchorcolor=green,
            citecolor=cyan% violet
            ]{hyperref}
\usepackage{cleveref}  
\usepackage{float}
\usepackage{graphicx}
\usepackage{epstopdf}
\usepackage{multirow}
\usepackage{caption,subcaption}
 \usepackage[utf8]{inputenc}
 \usepackage[ruled,vlined]{algorithm2e}
\usepackage{etoolbox}
 \usepackage{enumitem}
\makeatletter
\patchcmd{\@algocf@start}% <cmd>
  {-1.5em}% <search>
  {0pt}% <replace>
  {}{}% <success><failure>
\makeatother

\usepackage{xpatch}

\makeatletter
\xpatchcmd\SetKwInOut
  {\hangafter=1\parbox[t]}
  {\hangafter=1\justify\parbox[t]}
  {}{\fail}
\makeatother
\usepackage{ragged2e}

\newtheorem{theorem}{Theorem}[section]

\newtheorem{lemma}{Lemma}[section]

\newtheorem{corollary}{Corollary}[section]
\newtheorem{definition}{Definition}[section]
\newtheorem{remark}{Remark}[section]
\newtheorem{assumption}{Assumption}[section]
\newtheorem{example}{Example}[section]
\def\D{{\mathcal D}}
\def\K{{\mathcal K}} 
\def\S{{\mathcal S}}
\def\L{{\mathcal L}}
\def\F{{\mathcal F}}
\def\R{{\mathbb  R}}
\def\N{{\mathbb  N}}

\def\ctauk{\mathcal{C}^{\tau_{k+1}}}
\def\Z{{\mathbf  Z}}
\def\Y{{\mathbf  Y}}

\def\ba{\mathbf{a}}
\def\bb{\mathbf{b}}

\def\bw{\mathbf{w}}
\def\bx{\mathbf{x}}
\def\by{\mathbf{y}}
\def\bz{\mathbf{z}}
\def\blx{\overline{\mathbf{x}}}

\def\bpi{{\boldsymbol  \pi}}

\def\bg{{\boldsymbol  g}}
\def\blg{\overline{\boldsymbol  g}}
\def\GD{{\tt FedAvg}}

\def\GIA{{\tt FedGiA}}
\def\GIAD{{\tt FedGiA$_{\tt D}$}}
\def\GIAG{{\tt FedGiA$_{\tt G}$}}

\graphicspath{ {./images/} }
\ifCLASSINFOpdf
\else
\fi
\begin{document}
\title{{FedGiA{:} An Efficient Hybrid Algorithm for Federated Learning}}

\author{Shenglong Zhou  {and Geoffrey Ye Li, \textit{Fellow, IEEE}}% <-this % stops a space
%\thanks{This work is supported in part by the National Natural Science Foundation of China (12131004) and Beijing Natural Science Foundation (Z190002).}
\thanks{S.L. Zhou and G.Y. Li are with the ITP Lab,  Department of Electrical and Electronic Engineering, Imperial College London, London SW72AZ, United Kingdom.  E-mail: \{shenglong.zhou, geoffrey.li\}@imperial.ac.uk.}
\thanks{S.L. Zhou is also with the School of Mathematics and Statistics, Beijing Jiaotong University, Beijing 100044, China.}
% <-this % stops a space
%\thanks{Z.Y. Luo and N.H. Xiu are with the Department of Applied Mathematics, Beijing Jiaotong University, Beijing 100044,  People's Republic of China, e-mail: \{zyluo, nhxiu\}@bjtu.edu.cn}% <-this % stops a space
}


\markboth{}
{Shell \MakeLowercase{\textit{et al.}}: Bare Demo of IEEEtran.cls for IEEE Journals}



\maketitle

\begin{abstract}
Federated learning has shown its advances recently but is still facing many challenges, such as how algorithms save communication resources and reduce computational costs, and whether they converge. To address these critical issues, we propose a hybrid federated learning algorithm (\GIA) that combines the gradient descent and the inexact alternating direction method of multipliers. {The proposed algorithm is more communication- and computation-efficient than several state-of-the-art algorithms theoretically  and numerically. Moreover,  it also converges globally under mild conditions.  }
\end{abstract}


% \fi
\IEEEpeerreviewmaketitle

\section{Introduction}
\IEEEPARstart{F}{ederated}  learning (FL)  is burgeoning into an advanced approach in machine learning  presently due to the ability to deal with various issues like data privacy, data security, and data access  to heterogeneous data. Typical applications include  vehicular communications \cite{samarakoon2019distributed, pokhrel2020federated,
elbir2020federated, posner2021federated}, digital health 
\cite{rieke2020future}, and  smart manufacturing \cite{fazel2013hankel}, just naming a few. The earliest work for FL can be traced back to \cite{konevcny2015federated} in 2015  and \cite{konevcny2016federated} in 2016. It is still undergoing development and facing many challenges \cite{kairouz2019advances,li2020federated,qin2021federated}.
\subsection{Related work}
{\it Gradient descent-based learning.}  Lately, there is an impressive body of work on developing FL algorithms. One of the most popular approaches benefits from the stochastic gradient descent (SGD). The general framework is to run certain steps of SGD in parallel by clients  and then average the resulting parameters from clients by a central server once in a while.
Representatives of SGD family consist of the famous federated averaging ({\tt FedAvg} \cite{mcmahan2017communication}), its modified version \cite{xu2021learning}, and local SGD ({\tt LocalSGD} \cite{stich2018local, Lin2020Don}). Other state-of-the-art ones can be found in \cite{AsynchronousStochastic2017,  yu2019parallel, wang2021cooperative}. These algorithms execute global aggregation (or averaging)  periodically and thus can reduce the communication rounds (CR), thereby saving resources (e.g., transmission power and bandwidth in wireless communication) for real-world applications. 

However, to establish the convergence theory, most SGD algorithms assume that the local data is identically and independently distributed (i.i.d.), which is unrealistic for FL applications where data is usually heterogeneous. More details can be referred to the local SGD \cite{stich2018local}, K-step averaging SGD \cite{zhou2017convergence},  and cooperative SGD \cite{wang2021cooperative}.

A parallel line of research aims to investigate gradient descent (GD) based-FL algorithms. Since full data is used to construct the gradient, these algorithms do not impose assumptions on distributions of the involved data \cite{smith2018cocoa, LAG2018, wang2019adaptive, liu2021decentralized, tong2020federated}.  Nevertheless, strong conditions on the objective functions of the learning optimization problems are frequently assumed to guarantee convergence, such as the gradient Lipschitz continuity (also known as L-smoothness), convexity, or strong convexity.

{\it ADMM-based learning.} The alternating direction method of multipliers (ADMM) has been rapidly developed in theoretical and numerical aspects over the last few decades, with extensive applications into various disciplines. In particular,  there is a success of implementation of ADMM in distributed learning \cite{
zhang2018improving,zheng2018stackelberg,elgabli2020fgadmm}. Fairly recently, ADMM-based FL algorithms draw much attention due to its simple structure and easy implementation. These algorithms can be categorized into two classes: exact and inexact ADMM. The former aims at updating local parameters through solving sub-problems exactly, which hence brings more computational burdens for local clients \cite{ zhang2016dynamic, Li2017RobustFL, guo2018practical, huang2019dp, zhang2018recycled}.

Therefore, inexact ADMM provides a promising solution to reduce the computational complexity \cite{ding2019stochastic, Inexact-ADMM2021, ryu2022differentially, zhang2020fedpd}, where clients update their parameters via solving sub-problems approximately, thereby alleviating the computational burdens and accelerating the learning speed. Again, we shall emphasize that most of these algorithms need restrictive assumptions to ensure convergence. To overcome this, an algorithm from the primal-dual optimization perspective has been cast in \cite{zhang2020fedpd} and turns out to be a member of inexact ADMM-based FL algorithms. It is shown that the algorithm converges under  weaker assumptions. 
 %Finally, it is worth mentioning that ADMM is very useful in FL for the purpose of data privacy \cite{zhang2016dynamic, guo2018practical, zhang2018recycled, ding2019stochastic, huang2019dp,  ryu2022differentially}.

\subsection{Our contributions}

The main contribution of this paper is to develop a new FL  algorithm that is capable of saving communication resources, reducing computational burdens, and converging under relatively weak assumptions.

 \begin{table*} 
	{\renewcommand{\arraystretch}{1.75}\addtolength{\tabcolsep}{3pt}
	\caption{{Comparisons of different algorithms. Assumptions: \textcircled{1} Gradient Lipschitz continuity; \textcircled{2} Bounded level set; \textcircled{3} Strong convexity; \textcircled{4} KL property; \textcircled{5} Bounded gradient dissimilarity. Here, $\beta_1,\beta_2,$ and $\beta_3$ are defined in Remark \ref{remark:com}.}}\vspace{-3mm}
	\label{tab:com-algs}
	\begin{center}
		\begin{tabular}{lllccc } \hline
Algs.  &Ref.& Convergence rate & Communication rounds &	Assumptions	&	Computational complexity \\  
 &   &   &	 &	&  (for $k_0$ steps) \\ \hline
\multicolumn{6}{c}{Type-I convergence: $\|\nabla f(\bx^k)\|^2\leq \epsilon$}\\\hline 
{\tt FedPorx} &\cite{li2020federatedprox}&  $O\left( \frac{1}{k} \right)$ &$O\left(\frac{1}{\epsilon}\right)$&	 	\textcircled{1} \textcircled{5}  &	$O\Big(   (  \beta_3 +  n)mk_0\Big)$\\
{\tt FedAvg} &\cite{karimireddy2020scaffold}&  $O\Big(\sqrt{\frac{k_0}{k}}\Big)$ &$O\Big(\frac{1}{\epsilon^2}\Big)$&  	\textcircled{1} \textcircled{5} &	$O\Big(   (  \beta_1+   n)mk_0\Big)$\\
{\tt SCAFFOLD} &\cite{karimireddy2020scaffold}&  $O\Big(\sqrt{\frac{k_0}{k}}\Big)$ &$O\Big(\frac{1}{\epsilon^2}\Big)$&	 	\textcircled{1}  &	$O\Big(   (  \beta_2 +   n)mk_0\Big)$ 	\\
{\tt FedPD} &\cite{zhang2020fedpd}  &$O\Big(\frac{k_0}{k}\Big)$ &$O\Big(\frac{1}{\epsilon}\Big)$&	 	\textcircled{1}&	 $O\Big( (\beta_1  + n)mk_0\Big)$ \\
{\tt FedGiA}  & Our &$O\Big(\frac{k_0}{k}\Big)$ &$O\Big(\frac{1}{\epsilon}\Big)$&	\textcircled{1} & $O\Big( (\beta_1/k_0+ n)mk_0\Big)$  \\\hline
\multicolumn{6}{c}{Type-II convergence: $f(\bx^k)-f^*\leq \epsilon$}\\\hline
 {\tt LocalSGD} &\cite{stich2018local}&  $O\left( \frac{k_0}{k} \right)$ &$O\left(\frac{1}{\epsilon}\right)$&	 	\textcircled{1} \textcircled{3} \textcircled{5}  &	$O\Big(   (  \beta_2 +  n)mk_0\Big)$\\
 {\tt FedAvg}&\cite{li2019convergence}&  $O\left(\frac{k_0}{k}\right)$ &$O\left(\frac{1}{\epsilon}\right)$&	 	\textcircled{1}  \textcircled{3}&	 \\
 {\tt FedAvg} &\cite{karimireddy2020scaffold}&  $O\left(\frac{k_0}{k}\right)$ &$O\left(\frac{1}{\epsilon}\right)$&	 	\textcircled{1}  \textcircled{3} \textcircled{5}&	 \\ 
{\tt SCAFFOLD}& \cite{karimireddy2020scaffold}&  $O\left(\frac{k_0}{k}\right)$ &$O\left(\frac{1}{\epsilon}\right)$&	 	\textcircled{1}  \textcircled{3}& \\ 
 {\tt FedGiA}   & Our&$
 \arraycolsep=1pt\def\arraystretch{1.5}
 \left\{ \begin{array}{lll }
O(0),& \theta=0,\\ 
 O\Big(\left(\frac{\rho}{\rho+1}\right)^{\frac{k}{k_0}} \Big), & \theta\in(0,\frac{1}{2}],\\ 
 O\Big(\left(\frac{k_0}{k}\right)^{\frac{1}{2\theta-1}}\Big),& \theta\in(\frac{1}{2},1),
   \end{array}\right.$& $
   \arraycolsep=1pt\def\arraystretch{1.5}
   \left\{ \begin{array}{lll }
O(1),& \theta=0,\\[1ex]
 O\Big(\log_{\frac{\rho+1}{\rho}}(\frac{1}{\epsilon})\Big), & \theta\in(0,\frac{1}{2}],\\[1ex]
 O\Big( \frac{1}{\epsilon^{2\theta-1}} \Big), & \theta\in(\frac{1}{2},1),
   \end{array}\right.$ &	\textcircled{1} \textcircled{2}  \textcircled{4}  &	   \\ 
   {\tt FedGiA}   & Our&$
 O\Big(\left(\frac{\rho}{\rho+1}\right)^{\frac{k}{k_0}} \Big)$& $
 O\Big(\log_{\frac{\rho+1}{\rho}}(\frac{1}{\epsilon})\Big)$ &	\textcircled{1} \textcircled{3}  &	   \\ \hline
 		\end{tabular}
	\end{center}} \vspace{-5mm}
\end{table*}  
 I) The proposed algorithm, \GIA\ in Algorithm \ref{algorithm-CEADMM}, has a novel framework.  When iteration $k$ is a multiple of a given integer $k_0$, communication occurs between each client and the central server. At the same time, all clients are split into two groups randomly.  One group adopts the scheme of the inexact ADMM to update their parameters $k_0$ times, while the second group exploits the GD to update their parameters just once. In summary, \GIA\ possesses three advantages as follows.
 \begin{itemize}[leftmargin=10pt]
 \item It is communication-efficient since CR can be controlled by setting $k_0$. Our numerical experiments have shown that CR decline when $k_0$ increases, see Figure \ref{fig:effect-k0-diff}.

 \item It is computation-efficient, which results from two aspects.  All local clients take advantage of inexact updates  and the key item (i.e., the gradient) need to be computed only once for $k_0$ steps. The computational efficiency can be found in Table \ref{tab:com-algs} and in our numerical comparisons in Table \ref{tab:com-linear}.

 \item It is possible to cope with situations where a portion of clients are in bad conditions. The sever could put them into the second group where less effort is required to update their parameters.
\end{itemize}

{II) The assumptions to guarantee convergence are mild.  \GIA\ is proven to converge to a stationary point of \eqref{FL-opt} and to enjoy two types of convergence rate, as shown in Table \ref{tab:com-algs}.  For type-I convergence, \GIA\ achieves the sub-linear convergence rate (i.e., $O(k_0/k)$) only under the assumption of the gradient Lipschitz continuity. If we further assume the boundedness of a level set and   Kurdyka-Lojasiewicz (KL)  property, a weaker condition than the strong convexity, then it has the type-II convergence rates better than $O(k_0/k)$ enjoyed by the other algorithms.  However, it is worth mentioning that these rates of convergence established under the KL property are in a local sense, that is, \GIA\ converges with such rates when its generated sequence  is quite close to the stationary point.  Moreover, with the help of the strongly convexity,  other algorithms only converge sub-linearly while \GIA\ converges linearly, as shown in the last row in Table \ref{tab:com-algs}.  To summarize, in comparison with these algorithms in the table, \GIA\ can achieve the fastest convergence under the weakest conditions, thereby consuming the fewest CR. 
 }

% as shown in Theorem \ref{global-convergence-exact}..  with {a rate $O(k_0/k)$} only under two conditions: gradient Lipschitz continuity (also known for the L-smoothness in many publications)  and the boundedness of a level set, as shown in Theorem \ref{global-convergence-exact}. These conditions do not impose convexity or strong convexity. Hence, they are weaker than those used to establish convergence for  most current distributed learning and FL  algorithms. If we further assume the convexity, then  \GIA\  achieves the optimal solution, as shown in Corollary \ref{L-global-convergence}. 

\subsection{Organization and notation}
This paper is organized as follows.  In the next section, we introduce FL  and the framework of ADMM.  In  Section \ref{sec:ceadmm},  we present the algorithmic framework of \GIA\ and highlight its advantages. The global convergence and convergence rate are established in Section \ref{sec:convergence}.  We conduct some numerical experiments and comparisons with three leading solvers to demonstrate the performance of  \GIA\  in Section  \ref{sec:num}.  Concluding remarks are given in the last section.

We end this section with summarizing the notation that will be employed throughout this paper. We use  plain,  bold, and capital letters to present scalars, vectors, and matrices, respectively, e.g., $m , r,$ and $\sigma$ are scalars, $\bx, \bx_i$ and $\bx_i^k$  are vectors, $X $ and $X ^k$ are matrices. Let $\lfloor t\rfloor$ represent  the largest integer strictly smaller than $t+1$ and $[m]:=\{1,2,\ldots,m\}$ with `$:=$' meaning define.  In this paper, $\R^n$ denotes the $n$-dimensional Euclidean space equipped with the inner product $\langle\cdot,\cdot\rangle$ defined by $\langle\bx,\by\rangle:=\sum_i x_iy_i$. Let $\|\cdot\|$ be the Euclidean norm for vectors (i.e., $\|\bx\|^2=\langle\bx,\bx\rangle$) and Spectral norm for matrices, and $\|\cdot\|_H$ be the  weighted norm defined by  $\|\bx\|_H^2:=\langle H\bx,\bx\rangle$.
 Write the identity matrix as $I$ and a positive semidefinite matrix $A$ as $A\succeq 0$. In particular, $A\succeq B$ represents $A-B\succeq 0$.  A function, $f$, is said to be gradient Lipschitz continuous with a constant $r>0$ if 
 \begin{eqnarray}\label{Lip-r} 
\|\nabla f(\bx)-\nabla f(\bz  ) \|  \leq r\| \bx -\bz  \|.
  \end{eqnarray}
for any $\bx$ and $\bz$, where $\nabla  f(\bx)$  represents the gradient of $f$ with respect to $\bx$. Finally, throughout the paper, let
$$X :=(\bx_1,\bx_2,\ldots,\bx_m),\qquad\Pi:=(\bpi_1,\bpi_2,\ldots,\bpi_m)$$
%
%A function $f$ is said to be gradient Lipschitz continuous with a constant $r>0$ if for any $\bx$ and $\bz$,
% \begin{eqnarray}\label{Lip-r} 
%\|\nabla f(\bx)-\nabla f(\bz  ) \|  \leq r\| \bx -\bz  \|,
%  \end{eqnarray}
%where $\nabla  f(\bx)$ is the gradient of $f$ with respect to $\bx$. 

 \section{GD and inexact ADMM-based FL}
% \subsection{Learning optimization model}
Given $m$ clients, each client $i\in[m]$ has its local dataset $\D_i$ and loss function $f_i(\bx) := \frac{1}{d_i}\sum_{(\ba,b)\in\D_i} \ell_i(\bx; (\ba,b))$ on that $\D_i$, where $\ell_i(\cdot; (\ba,b)):\R^n\to\R$ is continuously differentiable and bounded from below, $d_i$ is the cardinality of $\D_i$,  and $\bx\in\R^n$ is the parameter to be learned. Below are two examples  used for our numerical experiments.

\begin{example}[Least square loss] \label{ex:lr} Suppose  client $i$ has dataset $\D_i=\{(\ba^i_1,b^i_1),\ldots,(\ba^i_{d_i},b^i_{d_i})\}$, where $\ba^i_j\in\R^n$, $b^i_j\in\R$.  Then the least square loss is
\begin{eqnarray} \label{least-squares}
 \arraycolsep=1.4pt\def\arraystretch{1.5}
\begin{array}{llll}
f_i(\bx)&=& \frac{1}{2d_i}\sum_{j=1}^{d_i} (\langle \ba^i_j,\bx\rangle-b^i_j)^2.
\end{array} 
\end{eqnarray}
\end{example}
\begin{example}[$\ell_2$ norm regularized logistic loss]    \label{ex:lg}
Similarly,  client $i$ has dataset $\D_i$ but with $b^i_j\in\{0,1\}$. The $\ell_2$ norm regularized logistic loss is given by 
\begin{eqnarray} \label{logist-loss}
 \arraycolsep=0pt\def\arraystretch{1.5}
\begin{array}{llll}
f_i(\bx)&=&%\sum_{(\ba,b)\in\D_i}\Big[\ln (1+e^{\langle\ba,\bx\rangle} )-b\langle\ba,\bx\rangle\Big]\\&=& 
\frac{1}{d_i}\sum_{j=1}^{d_i}\left(\ln (1{+}e^{\langle\ba^i_j,\bx\rangle} ){-}b^i_j\langle\ba^i_j,\bx\rangle\right) {+} \frac{\mu}{2d_i}\|\bx\|^2,
\end{array} 
\end{eqnarray}
where  $\mu>0$ is a penalty parameter.
\end{example}
The overall loss function can be  defined as 
\begin{eqnarray*}\begin{array}{lll}
f(\bx) :=   \frac{1}{m}\sum_{i=1}^{m}   f_i(\bx). 
\end{array}\end{eqnarray*}
Federated learning aims to learn a best parameter $\bx^*$  that attains the minimal overall loss, namely,
\begin{eqnarray}\label{FL-opt}
 \arraycolsep=1.4pt\def\arraystretch{1.5}
\begin{array}{lll}
 \bx^*:=\underset{\bx\in\R^n }{\rm argmin}~f(\bx).
\end{array}\end{eqnarray}
Since $f_i$ is bounded from below, we have
\begin{eqnarray}\label{FL-opt-lower-bound}\begin{array}{lll}
 f^*:=f(\bx^*)>-\infty.
\end{array}\end{eqnarray}
 By introducing auxiliary variables, $\bx_i=\bx$,  problem  \eqref{FL-opt} can be equivalently rewritten as
\begin{eqnarray}\label{FL-opt-ver1}
 \arraycolsep=1.4pt\def\arraystretch{1.5}
\begin{array}{lll}
 \underset{ \bx, X}{\min}&&  F(X):=\frac{1}{m}\sum_{i=1}^{m}  f_i(\bx_i),\\
 {\rm s.t.}&&\bx_i=\bx,~i\in[m].
\end{array}\end{eqnarray}
In this paper, we focus on the above problem.  It is easy to see that $f(\bx)=F(X)$ if $X=(\bx,\bx,\ldots,\bx)$.

\subsection{ADMM}
The backgrounds of ADMM can be referred to the earliest work \cite{gabay1976dual} and a nice book \cite{boyd2011distributed}.
To apply ADMM for  \eqref{FL-opt-ver1},  we introduce the augmented Lagrange function  defined by, 
\begin{eqnarray}\label{Def-L}
\arraycolsep=1pt\def\arraystretch{1.5}
\begin{array}{lll}
 \L(\bx,X ,\Pi) 
&:=& \sum_{i=1}^{m} L(\bx,\bx_i ,\bpi_i)\\
L(\bx,\bx_i ,\bpi_i)&:=& \frac{1}{m} f_i(\bx_i) +    \langle \bx_i-\bx, \bpi_i\rangle +   \frac{\sigma}{2} \|\bx_i-\bx\|^2,
\end{array}\end{eqnarray} 
where $\bpi_i\in\R^n,i\in[m]$ are the Lagrange multipliers and $\sigma >0$. The framework of ADMM for problem \eqref{FL-opt-ver1} is given as follows:   for an initialized point  $(\bx^0, X^0, \Pi^0)$, perform the following updates iteratively for every $k\geq0$,
\begin{eqnarray}\label{framework-ADMM}
 \arraycolsep=0.5pt\def\arraystretch{1.5}
\left\{\begin{array}{llll} 
  \bx^{k+1} &=&  \underset{\bx \in\R^n}{\rm argmin}  ~\L(\bx,X^{k}, \Pi^{k}) = \frac{1}{m}\sum_{i=1}^{m}  (\bx^{k}_i+ \frac{\bpi_i^{k}}{\sigma}),\\
    \bx^{k+1}_i &=&  \underset{\bx_i\in\R^n}{\rm argmin}  ~L(\bx^{k+1},\bx_i, \bpi_i^k ),~~ i\in[m], \\
  \bpi^{k+1}_i &=&    \bpi_i^{k} + \sigma(\bx_i^{k+1}-\bx^{k+1}),~~~~i\in[m].\\
\end{array} \right.
\end{eqnarray} 

\subsection{Stationary points}
To end this section, we present the optimality conditions of problems \eqref{FL-opt-ver1} and (\ref{FL-opt}).
\begin{definition}\label{def-sta} A point $(\bx^*, X^*,\Pi^*)$ is a stationary point of   problem  (\ref{FL-opt-ver1}) if it satisfies 
\begin{eqnarray}\label{opt-con-FL-opt-ver1}
 \arraycolsep=1.0pt\def\arraystretch{1.25}
  \left\{\begin{array}{rcll}
 \frac{1}{m}\nabla  f_i(\bx_i^*)+\bpi_i^* &=&   0, ~~&i\in[m],  \\  
 \bx_i^*-\bx^* &=&0,&i\in[m],\\ 
  \sum_{i=1}^{m} \bpi_i^* &=& 0.
\end{array} \right.
\end{eqnarray} 
A point $\bx^*$ is  a stationary point of   problem  (\ref{FL-opt}) if it satisfies 
\begin{eqnarray}\label{opt-con-FL-opt}
 \begin{array}{llll}
\nabla  f(\bx^*)=0.
\end{array}
\end{eqnarray} 
\end{definition}
Note that any locally optimal solution to  problem  \eqref{FL-opt-ver1} (resp. (\ref{FL-opt})) must satisfy \eqref{opt-con-FL-opt-ver1} (resp. \eqref{opt-con-FL-opt}). If  $f_i$ is convex for every $i\in[m]$, then a point  is a globally optimal solution to  problem  \eqref{FL-opt-ver1} (resp.(\ref{FL-opt})) if and only if it satisfies condition \eqref{opt-con-FL-opt-ver1} (resp. \eqref{opt-con-FL-opt}).   Moreover, it is easy to see that a stationary point $(\bx^*, X^*,\Pi^*)$  of   problem  \eqref{FL-opt-ver1} indicates
\begin{eqnarray*} %\label{grad-x-*=0}
\begin{array}{llll}
~~~~\nabla   f(\bx^*) =\frac{1}{m}\sum_{i=1}^{m}  \nabla   f_i(\bx^*) 
=-\frac{1}{m} \sum_{i=1}^{m}  \bpi_i^*=0.\end{array}
\end{eqnarray*} 
That is, $\bx^*$ is also a stationary point of the   problem  \eqref{FL-opt}.  

\section{Algorithmic Design}\label{sec:ceadmm}
The framework of ADMM in \eqref{framework-ADMM} encounters three drawbacks in reality. (i) It repeats the three updates at every step, leading to communication inefficiency. In FL,  the framework manifests that local clients and the central server have to communicate at every step. However, frequent communications would come at a huge price, such as a long learning time and large amounts of resources. (ii) Solving the second sub-problem in  \eqref{framework-ADMM} would incur expensive computational cost as it generally does not admit a closed-form solution.   (iii) In real applications, some clients may suffer from bad conditions (e.g., limited computational capacity), which leads to computational difficulties. It is necessary to leave them more time to update their parameters.  Therefore, to overcome these drawbacks, we cast a new algorithm in  Algorithm \ref{algorithm-CEADMM}, where 
\begin{eqnarray*}   
\arraycolsep=1.5pt\def\arraystretch{1.5}
\begin{array}{lllrll}
 \overline\bg_i^{k+1}&:=&\frac{1}{m}\nabla f_i(\bx^{\tau_{k+1}}).
    \end{array}  
 \end{eqnarray*} 
The merits of Algorithm \ref{algorithm-CEADMM} are highlighted as follows.

\begin{algorithm} 
\SetAlgoLined
{\noindent \justifying Given an integer $k_0>0$ and a constant $\sigma>0$, every client $i$ initializes $H_i\succeq0, \bx_i^0$, $\bpi_i^0$ and $\bz^{0}_i =   \bx_i^{0} + {\bpi^{0}_i}/{\sigma}$,  $i\in[m]$. %The  server selects all clients, namely, $\M_{0}=[m]$. 
Let $\tau_k$ be a function of $k$ as $\tau_k:=\lfloor k/k_0 \rfloor.$ }

\For{$k=0,1,2,3,\ldots$}{

\If{$k\in\K:=\{0,k_0,2k_0,3k_0,\ldots\}$}{
\vspace{1mm}
  \underline{\it Weights upload: (Communication occurs)} 
  
  {\justifying \noindent All clients  upload  $\{\bz^{k}_1,\ldots,\bz_m^k\}$ to the server.}  
 
\vspace{1mm}
\underline{\it Global aggregation:} 

The server calculates average parameter  $\bx^{\tau_{k+1}}$ by 
 \begin{eqnarray}\label{ceadmm-sub1}
 \begin{array}{llll}
\bx^{\tau_{k+1}} =  \frac{1}{m}\sum_{i=1}^m   \bz_i^{k}.
\end{array}
\end{eqnarray}

\underline{\it Weights broadcast: (Communication occurs)}

The server  broadcasts  $\bx^{\tau_{k+1}}$ to all clients.

\vspace{1mm}
\underline{\it Clients selection:} 

{\justifying \noindent The server randomly selects a new set $\ctauk\subseteq [m]$ of clients for training in the next round.}
 
} 

\For{every $i\in\ctauk$}{
\underline{\it Local update:} Client $i$ updates its parameters by  
\begin{eqnarray} 
     \label{ceadmm-sub2}
\hspace{-18mm}&& \begin{array}{llll}
\bx^{k+1}_i  
 =    \bx^{\tau_{k+1}} - (H_i/m  + \sigma  I)^{-1} (\overline\bg_i^{k+1}+ \bpi_i^{k}), 
    \end{array}\\[1ex] 
\label{ceadmm-sub3}  
\hspace{-18mm}&&\begin{array}{llll}   
\bpi^{k+1}_i =   \bpi_i^{k} +\sigma(\bx_i^{k+1}-\bx^{\tau_{k+1}}),  
\end{array}\\[1ex]
\label{ceadmm-sub4}  
\hspace{-18mm}&&\begin{array}{llll}   
\bz^{k+1}_i =   \bx_i^{k+1}+\bpi^{k+1}_i/\sigma.  
\end{array}
\end{eqnarray} }
\For{every $i\notin\ctauk$}{
\underline{\it Local invariance:} Client $i$ keeps  parameters by   
\begin{eqnarray} 
\label{ceadmm-sub6}
\hspace{-18mm}&& \begin{array}{llll}
\bx^{k+1}_i  \equiv   \bx^{\tau_{k+1}}, 
    \end{array}\\[1ex] 
\label{ceadmm-sub7}  
\hspace{-18mm}&&\begin{array}{llll}   
\bpi^{k+1}_i \equiv - \overline\bg_i^{k+1} ,  
\end{array}\\[1ex]
\label{ceadmm-sub8}  
\hspace{-18mm}&&\begin{array}{llll}   
\bz^{k+1}_i \equiv  \bx^{\tau_{k+1}} - \overline\bg_i^{k+1}/\sigma.  
\end{array}
\end{eqnarray} 
}

}
\caption{FL via  GD and inexact ADMM ({\tt FedGiA}) \label{algorithm-CEADMM}}
\end{algorithm}

\subsection{Communication efficiency}
Algorithm \ref{algorithm-CEADMM} shows that communications only occur  when $k\in\K=\{0,k_0,2k_0,\ldots\},$ where $k_0$ is a predefined positive integer. Therefore, CR can be reduced if setting a big $k_0$, thereby  saving the cost vastly.  In fact, such an idea has been extensively used in literature \cite{stich2018local,Lin2020Don,AsynchronousStochastic2017,  yu2019parallel,  wang2021cooperative}.

\subsection{Fast computation} We update $\bx^{k+1}_i$ by \eqref{ceadmm-sub2} instead of solving the second sub-problem in  \eqref{framework-ADMM}. It can accelerate the computation for local clients significantly, as the computation is relatively cheap if $H_i$ is chosen properly (e.g, diagonal matrices).  We point out that \eqref{ceadmm-sub2} is a result of 
\begin{eqnarray} 
\label{iceadmm-sub2-0}
 \arraycolsep=1pt\def\arraystretch{1.5}
\begin{array}{llll}
\bx^{k+1}_i  
& = &  {\rm argmin}_{\bx_i}~ (1/m)h_i(\bx_i; \bx^{\tau_{k+1}})\\
&& + \langle \bx_i-\bx^{\tau_{k+1}}, \bpi_i^{k}\rangle +\frac{\sigma }{2}\|\bx_i-\bx^{\tau_{k+1}}\|^2\\
&=&  \bx^{\tau_{k+1}} - ( H_i/m  + \sigma  I)^{-1} (\blg_i^{k+1}+ \bpi_i^{k}),
    \end{array}  
\end{eqnarray}  
where $h_i(\bx_i;\bz)$ is an approximation of $f_i(\bx_i )$, namely,
\begin{eqnarray} 
\label{f-i-h-i} 
 \arraycolsep=1pt\def\arraystretch{1.5}
 \begin{array}{lll} 
h_i(\bx_i;\bz){:=}f_i(\bz) +  \langle \nabla f_i(\bz), \bx_i-\bz\rangle + \frac{1}{2}  \| \bx_i-\bz\|^2_{H_i}.
%&=:&h_i(\bx_i;\bx_i^k,H_i).
    \end{array}
\end{eqnarray}
{The fast computation comes from two aspects. First, we take advantage of the inexact updates, where $(H_i/m  + \sigma  I)^{-1}$ needs to be computed only once for the entire learning as it is independent of $k$. In addition, one can observe that for every consecutive $k_0$ iterations,  $\blg_i^{k+1}=\frac{1}{m}\nabla f_i(\bx^{\tau_{k+1}})$ needs to be calculated  only once due to $\tau_{k+1}\equiv\tau_{sk_0}$ for any $k=sk_0,sk_0+1,\cdots,sk_0+k_0-1.$ Overall, such strategies allow clients to update their parameters quickly.}

\subsection{Mixed updates} 
At every $k\in\K$ in  Algorithm \ref{algorithm-CEADMM}, all clients are divided into two groups. For clients in $\ctauk$, they update their parameters $k_0$ times iteratively based on the inexact ADMM, while for clients outside $\ctauk$, they update their parameters just once based on the GD. {The motivation for using such a mixed scheme is twofold. 
\begin{itemize}[leftmargin=10pt]
\item In practice, it is unlikely to equip all clients with strong computational capacity. So, this scheme allows clients outside $\ctauk$ (corresponding to clients with weak computational capacity) to have more time to update their parameters by \eqref{ceadmm-sub6}-\eqref{ceadmm-sub8} just once, which does not require much computational endeavour. 
\item In addition, the scheme would avoid scenarios where the training for some clients is insufficient. If we only let clients in $\ctauk$ update their parameters and clients outside $\ctauk$ keep their parameters unchanged, then there is a small possibility where the selected times of some clients are not enough to train a good shared parameter. For instance, a worst case is that some clients are never chosen to join in the training. Then the trained parameter may not be appropriate for them as the training does not use their data at all.
\item As mentioned above, for every consecutive $k_0$ steps,  all (selected or non-selected) clients carry out the update at least once, namely all clients join in the training. Thanks to this, the total objective function values of the generated sequence can be guaranteed to decrease with a declining scale at every step, as shown by Lemma \ref{lemma-decreasing-0}. This thus accelerates the convergence so as to reduce CR, see the reduced CR in Table \ref{tab:com-algs}. 
\item One can check that \eqref{ceadmm-sub2}-\eqref{ceadmm-sub4} imply 
$$\begin{array}{llll}   
\bz^{k+1}_i =   \bx_i^{k+1}- \overline\bg_i^{k+1}  /\sigma -H_i(\bx^{k+1}_i - \bx^{\tau_{k+1}})/(m\sigma).  
\end{array}$$
Therefore, the inexact ADMM update can be deemed as a proximal GD update if $H_i\neq0$ and a GD update if $H_i=0$.
\end{itemize}}

We would like to point out that {\tt FedAvg} \cite{mcmahan2017communication,li2019convergence},  {\tt FedProx} \cite{li2020federatedprox}, and {\tt FedADMM} \cite{zhou2023federated} select partial devices to join in the training in each communication round. That is, if $k\in\K$, they randomly select a subset $\ctauk$ of clients.  Only clients in $\ctauk$ update their parameters and the rest clients remain unchanged. However, differing from that, {\tt FedGiA} let clients outside $\ctauk$ also update their parameters but only once for every consecutive $k_0$ steps.

\section{Convergence Analysis}\label{sec:convergence}
We first present all the assumptions used to establish the convergence properties. 
\begin{assumption}\label{ass-fi} Every $f_i, i\in[m]$ is gradient Lipschitz continuous with a constant $r_i>0$. 
\end{assumption}
Assumption \ref {ass-fi} implies that there is always a $\Theta_i$ satisfying $r_iI\succeq \Theta_i \succeq 0$ such that
\begin{eqnarray} \label{grad-lip-theta}
 \arraycolsep=0pt\def\arraystretch{1.5}
 \begin{array}{l}
f_i(\bx_i)\leq   f_i(\bz_i )+\langle \nabla  f_i(\bz_i ), \bx_i-\bz_i \rangle + \frac{1}{2}\| \bx_i-\bz_i \|^2_{\Theta_i},
\end{array}
\end{eqnarray} 
for any $\bx_i,\bz_i\in\R^n$. Apparently, many $\Theta_i$s satisfy  the above condition (e.g.,  $\Theta_i=r_iI$). In the subsequent convergence analysis, we suppose that every client $i$  chooses  $H_i=\Theta_i$. 
\begin{assumption}\label{ass-fi-level} The following level set    is bounded,
\begin{eqnarray} \label{level-set-S} 
  \begin{array}{l}
  \S:=\{\bx\in\R^n: f(\bx)\leq \L(\bx^{0},X^0,\Pi^0)\}.
     \end{array}
  \end{eqnarray} 
\end{assumption}
 We point out that the boundedness of the level set is frequently used in establishing the convergence properties of optimization algorithms. The main purpose is to bound the generated sequence. There are many functions satisfying this condition, such as the coercive functions.\footnote{A continuous function $f:\R^n\mapsto \R$  is coercive if $ f(\bx)\rightarrow+\infty$ when $\|\bx\|\rightarrow +\infty$.}  {It is noted that  the boundedness of the level set and the convexity does not imply each other. For example, $|t^2{-}1|$ is non-convex on $\R$ but has the bounded level set while $e^t$ is convex on $\R$ but does not have a bounded level set. However, if a function is strongly convex, then all its level sets are bounded.}  
{\begin{assumption}\label{ass-fi-KL} Every $f_i, i\in[m]$ is a KL function, defined by Definition \ref{def-KL-func}. 
\end{assumption}
KL functions are general enough \cite{bolte2014proximal} and include most commonly used functions, such as the  real polynomial functions, logistic loss function, norm functions, indicator functions of semi-algebraic sets (e.g., cone of positive semi-definite matrices, Stiefel manifolds, and constant rank matrices). Moreover, most convex
functions (e.g., functions in \eqref{least-squares} and \eqref{logist-loss} and strongly convex functions) in finite dimensional applications satisfy the KL property as well. Hence, there are various convex and non-convex functions belong to KL functions.}
  \subsection{Global convergence}\label{sec:gc}
For notational convenience, hereafter we  let   $\ba^{k} \rightarrow \ba$ stand for $\lim_{k\rightarrow\infty} \ba^{k} = \ba$ and denote
\begin{eqnarray} \label{def-L-r}
 \arraycolsep=0pt\def\arraystretch{1.5}
 \begin{array}{ll}
 \Z^k:=(\bx^{\tau_{k}},X^{k},\Pi^{k}),&  \Z^*:=(\bx^{*},X^{*},\Pi^{*}),\\
 \Z^{\infty}:=(\bx^{\infty},X^{\infty},\Pi^{\infty}),\qquad&  r:=\max_{i\in[m]} r_i,\\
\triangle\bx_i^{k+1}:=\bx_i^{k+1}-\bx_i^{k},&\triangle\bx^{\tau_{k+1}}:=\bx^{\tau_{k+1}}-\bx^{\tau_{k}}.
\end{array} 
\end{eqnarray}
With the help of Assumption \ref{ass-fi}, our first result shows the decreasing property of the objective function values of the generated sequence.
\begin{lemma}\label{lemma-decreasing-0}  Let $\{\Z^{k} \}$ be the sequence generated by Algorithm \ref{algorithm-CEADMM} with $H_i=\Theta_i,i\in[m]$ and $\sigma\geq6r/m$. If Assumption \ref{ass-fi} holds, then the following statements are true.
\begin{itemize} 
\item[i)] For any $k\geq0$, 
\begin{eqnarray} 
 \label{decreasing-property-0}    
\begin{array}{lll}
 {\L}(\Z^{k+1}) -{\L} (\Z^{k})
 \leq   - \frac{\sigma }{24}\varpi_{k+1}, 
    \end{array}  
 \end{eqnarray} 
 where 
 \begin{eqnarray} 
 \label{def-Y-k1} 
  \arraycolsep=1.4pt\def\arraystretch{1.5}  
\begin{array}{lll}
\varpi_{k+1}:= \sum_{i=1}^m ( \| \triangle\bx^{\tau_{k+1}}\|^2  +  \| \triangle\bx_i^{k+1}\|^2). 
    \end{array}  
 \end{eqnarray} 
 \item[ii)] For any $s\in \K$, 
 \begin{eqnarray} 
 \label{grad-decreasing-property-0}   
  \arraycolsep=1.4pt\def\arraystretch{1.5}
\begin{array}{lll}
 \|\nabla \L(\Z^{s+1})\| \leq  (2\sigma+1)  \sqrt{m \varpi_{s+1}}.
    \end{array}  
 \end{eqnarray} 
 \end{itemize} 
\end{lemma}  
The above lemma allows us to prove whole sequences $\{\L(\Z^{k})\}$, $\{ F(X^{k})  \}$, and $\{f (\bx^{\tau_{k}})\}$ converge. 
\begin{theorem}\label{global-obj-convergence-exact}   Let $\{ \Z^{k}\}$ be the sequence generated by Algorithm \ref{algorithm-CEADMM} with $H_i=\Theta_i, i\in[m]$ and  $\sigma \geq 6r/m$. The following results hold under Assumption \ref{ass-fi}.
 \begin{itemize}
 \item[i)] 
  Three  sequences $\{\L(\Z^{k})\}$, $\{ F(X^{k})  \}$, and $\{f (\bx^{\tau_{k}})\}$ converge to the same value, namely,
   \begin{eqnarray}  \label{L-local-convergence-limit-1}
   \arraycolsep=1.4pt\def\arraystretch{1.5}
   \begin{array}{lll}
  \lim\limits_{k\rightarrow \infty}  \L(\Z^{k}) =  \lim\limits_{k\rightarrow \infty} F(X^{k})  = \lim\limits_{k\rightarrow \infty} f(\bx^{\tau_{k}}).
    \end{array} 
 \end{eqnarray} 
 \item[ii)]  $\nabla F(X^{k})$ and $\nabla f(\bx^{\tau_{k}})$ eventually vanish, namely, 
    \begin{eqnarray}  \label{L-local-convergence-limit-grad}
   \arraycolsep=1.4pt\def\arraystretch{1.5}
   \begin{array}{lll}
  \lim\limits_{k\rightarrow \infty}\nabla F(X^{k})  = \lim\limits_{k\rightarrow \infty} \nabla f(\bx^{\tau_{k}}) =0.
    \end{array} 
 \end{eqnarray} 
 \end{itemize}
 \end{theorem}  
% {We point out that the establishments of Lemma \ref{lemma-decreasing-0} and Theorem \ref{global-obj-convergence-exact} do not relied on the random selection of $\ctauk$, namely, no assumptions on its selection. A possible reason can be given as follows: It is noted that for every consecutive $k_0$ steps, all (selected or non-selected) clients update their parameters at least once. Thanks to this,  the total objective function value, $\L(\Z^k)$, is guaranteed to be non-increasing as \eqref{decreasing-property-0}, thereby leading the convergence property as Theorem \ref{global-obj-convergence-exact}. In other words, like {\tt FedAvg} or {\tt FedProx}, if clients outside $\ctauk$ do not carry out any update, then the decreasing property of sequence $\L(\Z^k)$ may not be preserved.} 
Besides the convergence of the objective function values in the above theorem,  we would like to see the convergence
performance of sequence $\{\Z^{k}\}$ itself in the following theorem.    
  \begin{theorem}\label{global-convergence-exact}   Let $\{\Z^{k}\}$ be the sequence generated by Algorithm \ref{algorithm-CEADMM} with $H_i=\Theta_i, i\in[m]$ and  $\sigma \geq 6r/m$. The following results hold under Assumptions \ref{ass-fi} and \ref{ass-fi-level}.
 \begin{itemize}
\item[i)]  Then sequence $\{\Z^{k}\}$ is bounded, and any its accumulating point,  $\Z^{\infty}$, is a stationary point of    (\ref{FL-opt-ver1}), where $\bx^{\infty}$ is a stationary point of   (\ref{FL-opt}).
\item[ii)] If further assume that either $\bx^{\infty}$ is isolated or Assumption \ref{ass-fi-KL} holds, then whole sequence $\{\Z^{k} \}$ converges to $\Z^{\infty}$. 
  \end{itemize}
   \end{theorem}  
\begin{remark}{We give some comments on the conditions in Theorem \ref{global-convergence-exact}.
\begin{itemize}[leftmargin=10pt]
\item According to the proof, apart from the choice of $H_i=\Theta_i, i\in[m]$, any matrix $H_i$ satisfying $r_iI\succeq H_i\succeq 0$ suffices to ensure the convergence property. This means  $H_i$ can be chosen flexibly. However, for the sake of reducing the computational complexity, equation (\ref{ceadmm-sub2}) suggests that we should set $H_i$ to be a matrix enabling the easy calculation of $(H_i/m  + \sigma  I)^{-1}$. Typical choices include the gram matrix or diagonal matrix, see Table \ref{tab:choice-Hi}.
\item Condition $\sigma \geq 6r/m$ is essential.  First, since the problem  may be non-convex, a properly large $\sigma$ can guarantee the strong convexity of \eqref{iceadmm-sub2-0} so as to ensure a unique solution for the case of $H_i=0$. In addition, the big value can prevent over-updating for $\bx^{k+1}_i$, namely, avoiding it stepping too further from global parameter $\bx^{\tau_{k+1}}$. 
\item It is noted that if $f$ is locally strongly convex at $\bx^{\infty}$, then $\bx^{\infty}$ is unique and hence is isolated. However, being isolated is a weaker assumption than locally strong convexity. In addition, there are extensive functions satisfying Assumption \ref{ass-fi-KL} since KL functions are general enough. 
\end{itemize}}
\end{remark}
It is worth mentioning that the establishment of Theorem \ref{global-convergence-exact} does not require the convexity of $f_i$ or $f$, because of this, the sequence is guaranteed to converge to the stationary point of problems (\ref{FL-opt-ver1}) and (\ref{FL-opt}). In this regard, if we further assume the convexity of $f$, then the sequence is capable of converging to the optimal solution to problems (\ref{FL-opt-ver1}) and (\ref{FL-opt}), which is stated by the following corollary.
% \subsection{Convex case}
\begin{corollary}\label{L-global-convergence}Let $\{\Z^{k}\}$ be the sequence generated by Algorithm \ref{algorithm-CEADMM} with $H_i=\Theta_i, i\in[m]$ and  $\sigma \geq 6r/m$.    The following results hold under Assumptions \ref{ass-fi} and \ref{ass-fi-level}, and the convexity of $f$.
\begin{itemize}
\item[i)] Three  sequences $\{\L(\Z^{k})\}$, $\{ F(X^{k})\}$, and $\{ f (\bx^{\tau_{k}})\}$ converge to the optimal function value of (\ref{FL-opt}), namely
 \begin{eqnarray}  \label{L-global-convergence-limit}
   \arraycolsep=1.4pt\def\arraystretch{1.5}
   \begin{array}{lll}
 \lim\limits_{k\rightarrow \infty}  \L(\Z^{k})  = \lim\limits_{k\rightarrow \infty}  F(X^{k})= \lim\limits_{k\rightarrow \infty} f(\bx^{\tau_{k}}) =  f^*.
    \end{array} 
 \end{eqnarray}   
 \item[ii)] Any accumulating point  $\Z^{\infty} $ of   sequence  $\{\Z^{k}\}$ is an optimal solution to  (\ref{FL-opt-ver1}) and $\bx^{\infty}$ is an optimal solution to   (\ref{FL-opt}). 
 
\item [iii)]  If further assume $f$ is strongly convex. Then whole sequence  $\{\Z^{k}\}$ converges to unique optimal solution $\Z^*$  to  (\ref{FL-opt-ver1}) and $\bx^*$ is the unique optimal solution to  (\ref{FL-opt}).  
\end{itemize} 
\end{corollary} 
  \begin{remark}Regarding the assumption in Corollary \ref{L-global-convergence}, $f$ being strongly convex does not require the strong convexity for every $f_i,i\in[m]$. If one  $f_i$ is strongly convex and the remaining is convex, then $f=\sum_{i=1}^{m} w_if_i$ is strongly convex.    Moreover, the strongly convexity suffices to the boundedness of level set  $\S$. Therefore, under the strongly convexity, the assumption on the boundedness of $\S$ can be exempted.
  \end{remark}
  
%  \subsection{Non-convex case}
 \subsection{Convergence rate}\label{sec:cs}
In this part, we investigate the convergence rate of Algorithm \ref{algorithm-CEADMM}. The following result states that the minimal value among $\|\nabla f (\bx^{\tau_{j}})\|^2, j\in[k]$  vanishes with rate $O(rk_0/k)$.

\begin{theorem}\label{complexity-thorem-gradient}   Let $\{(\bx^{\tau_{k}},X^{k},\Pi^{k})\}$ be the sequence generated by Algorithm \ref{algorithm-CEADMM} with $H_i=\Theta_i, i\in[m]$ and  $\sigma \geq 6r/m$. If Assumption \ref{ass-fi} holds, then  it follows
       \begin{eqnarray*}
  \arraycolsep=1.4pt\def\arraystretch{1.5}
  \begin{array}{lllll}
    {\min}_{j\in[k]}  \| \nabla f (\bx^{\tau_{j}})\|^2 
 \leq  \frac{100m\sigma k_0}{k}   (\L(\Z^0)-f^*).
   \end{array}
  \end{eqnarray*}  
 \end{theorem}
The establishment of such a convergence rate only requires  the assumption of gradient Lipschitz continuity, namely, Assumption \ref{ass-fi}. Moreover, since $\sigma = t r/m$, we have
       \begin{eqnarray*}
  \arraycolsep=1.4pt\def\arraystretch{1.5}
  \begin{array}{lllll}
  {\min}_{j\in[k]}  \| \nabla f (\bx^{\tau_{j}})\|^2  = O(\frac{rk_0}{k}).
   \end{array}
  \end{eqnarray*} 
This is what we expected. The larger $k_0$ is, the more iterations is required to converge. 
\begin{remark}\label{remark-com} Theorem \ref{complexity-thorem-gradient} hints that Algorithm \ref{algorithm-CEADMM} can be terminated if   
       \begin{eqnarray}\label{stopping-tol}
  \arraycolsep=0pt\def\arraystretch{1.5}
  \begin{array}{lllll}
\| \nabla f (\bx^{\tau_{k}})\|^2 \leq   \epsilon,
   \end{array}
  \end{eqnarray} 
  where $\epsilon$ is a given tolerance. Therefore,  after 
         \begin{eqnarray}\label{stopping-iter}
  \arraycolsep=0pt\def\arraystretch{1.5}
  \begin{array}{lllll}
k =  \left\lfloor \frac{\rho k_0(\L(\Z^0)-f^*) }{\epsilon} \right\rfloor=O(\frac{rk_0}{\epsilon})
   \end{array}
  \end{eqnarray} 
  iterations,  Algorithm \ref{algorithm-CEADMM} meets  (\ref{stopping-tol}) and the total CR are
         \begin{eqnarray}\label{stopping-iter}
  \arraycolsep=0pt\def\arraystretch{1.5}
  \begin{array}{lllll}
CR:=\left\lfloor \frac{2k}{k_0} \right\rfloor  = \left\lfloor \frac{2\rho(\L(\Z^0) -f^*) }{\epsilon}  \right\rfloor=O(\frac{r}{\epsilon}).
   \end{array}
  \end{eqnarray}  
 \end{remark}
 {If we further assume the KL property, then we can achieve a better convergence rate as follows.
\begin{theorem}\label{complexity-thorem-F-F}    Let $\{\Z^{k}\}$ be the sequence generated by Algorithm \ref{algorithm-CEADMM} with $H_i=\Theta_i, i\in[m]$ and  $\sigma \geq 6r/m$, and $\Z^\infty$ be its limit. Suppose that Assumptions \ref{ass-fi} and \ref{ass-fi-level} hold, and Assumption \ref{ass-fi-KL} hold with a desingularizing function (see Definition \ref{def-Desingularizing}) $\varphi(z)=\frac{\sqrt{c}}{1-\theta} z^{1-\theta}$, where $c>0$, $\theta\in[0,1)$, then the following rates of convergence hold. 
\begin{itemize}
\item[i)] If $\theta=0$, then there is a $k_1\in\K$ such that
       \begin{eqnarray*}
  \arraycolsep=1.4pt\def\arraystretch{1.5}
  \begin{array}{lllll}
f(\bx^{\tau_{k}})\equiv  f^{\infty}:=f(\bx^{\infty}),~~
   \end{array}
  \end{eqnarray*} 
for all $k(\in \K)\geq k_1$.
\item[ii)] If $\theta\in(0,1/2]$, then there is a $k_2{\in}\K$ and $c_2{>}0$ satisfying
       \begin{eqnarray*}
  \arraycolsep=1.4pt\def\arraystretch{1.5}
  \begin{array}{lllll}
f(\bx^{\tau_{k}})- f^\infty \leq c_2\left(\frac{\rho}{\rho+1}\right)^{\frac{k-k_2}{k_0}}, 
   \end{array}
\end{eqnarray*} 
for all $k(\in \K)\geq k_2$, where $\rho:= {24 m c(2\sigma+1) ^2}/{\sigma}$.
\item[iii)] If $\theta\in(1/2,1)$, then there is a $k_3{\in}\K$ and $c_3{>}0$ satisfying
  \begin{eqnarray*} 
   \arraycolsep=1pt\def\arraystretch{1.75}   
\begin{array}{lcl}
f(\bx^{\tau_{k}})- f^\infty 
 \leq  \left(\frac{c_3k_0}{k-k_3}\right)^{\frac{1}{2\theta-1}},
    \end{array}  
 \end{eqnarray*}
 for all $k(\in \K)\geq k_3$. 
   \end{itemize}
 \end{theorem}
 Note that if $\theta=0$, the convergence result means that the algorithm can terminate within finitely many steps. In such a case we write $$f(\bx^{\tau_{k}})- f^\infty = O(0).$$  If $\theta\in(0,1/2]$, then the convergence rate is linear. If $\theta\in(1/2,1)$, then $\frac{1}{2\theta-1}\in(1,\infty)$ and thus the convergence rate is at least sub-linear. It is worth mentioning that all semi-algebraic functions (see \cite[Definition 5]{bolte2014proximal}) satisfy KL property with  $\varphi(z)=\frac{\sqrt{c}}{1-\theta} z^{1-\theta}$. Typical semi-algebraic functions include  $\|\bx\|_p:=\sum(|x_i|^p)^{1/p}, p\geq0$, real polynomial functions, and those functions in \cite[Example 2-4]{bolte2014proximal}. }
 
{In particular, if every $f_i$ is strongly convex, then Assumptions \ref{ass-fi-level} and \ref{ass-fi-KL} hold with $\varphi(z)=2\sqrt{cz}$ (see \cite[Example 6]{bolte2014proximal}), namely $\theta=1/2$ in Theorem \ref{complexity-thorem-F-F}, thereby leading to the linear convergence rate. 
\begin{remark}\label{remark-com-new} Based on Theorem \ref{complexity-thorem-F-F}, after  
         \begin{eqnarray}\label{stopping-iter}
  \arraycolsep=1pt\def\arraystretch{1.5}
 k =  \left\{ \begin{array}{lll }
O(k_0),& \theta=0,\\
 O\left(k_0\log_{\frac{\rho+1}{\rho}}(\frac{1}{\epsilon})\right),~~ & \theta\in(0,\frac{1}{2}],\\
 O\left( \frac{k_0}{\epsilon^{2\theta-1}} \right),~~ & \theta\in(\frac{1}{2},1),
   \end{array}\right.
  \end{eqnarray} 
iterations,  Algorithm \ref{algorithm-CEADMM}  meets  \begin{eqnarray}\label{stopping-acc}
  \arraycolsep=0pt\def\arraystretch{1.5}
  \begin{array}{lllll}
f(\bx^{\tau_{k}})- f^\infty  \leq   \epsilon.
   \end{array}
  \end{eqnarray} 
Hence, to achieve such an accuracy, the total CR are
         \begin{eqnarray}\label{stopping-iter}
           \arraycolsep=1pt\def\arraystretch{1.5}
CR:=\left\lfloor \frac{2k}{k_0} \right\rfloor  =  \left\{ \begin{array}{lll }
O(1),& \theta=0,\\
 O\left( \log_{\frac{\rho+1}{\rho}}(\frac{1}{\epsilon})\right),~~ & \theta\in(0,\frac{1}{2}],\\
 O\left( \frac{1}{\epsilon^{2\theta-1}} \right),~~ & \theta\in(\frac{1}{2},1).
   \end{array}\right.
  \end{eqnarray}  
 \end{remark}}
 \begin{remark}\label{remark:com}{We summarize several state-of-the-art algorithms and compare their performance in Table \ref{tab:com-algs}. Here, for $i\in[m]$, let $c_i^1$, $c_i^2$, and $c_i^3$  be  computational complexity of computing the gradient of $f_i$, the stochastic gradient of $f_i$, and an optimization problem associated with $f_i$, respectively. Then $\beta_t:=\max_{i\in[m]} c_i^t, t=1,2,3$.
 For type-I convergence, although  {\tt FedGiA} and {\tt FedPD}   have the best convergence rate $O(k_0/k)$ under the weakest assumption.   For the type-II convergence since the strong convexity implies the bounded level set and KL property, {\tt FedGiA}  converges with a rate better than sub-linear rate $O(k_0/k)$ for all other algorithms under the weakest assumptions. Moreover, it converges linearly if the strong convexity holds.}
 
{ 
 In addition to the best convergence results, {\tt FedGiA} also has a low computational complexity. In Table \ref{tab:com-algs}, we present the computational complexity for all clients to update their parameters $k_0$ times, that is, the computational complexity of local computation for consecutive $k_0$ steps between two communications. Note that the complexity for {\tt FedGiA} presented in the table is under the choice of $H_i$ being chosen as a diagonal matrix. From (\ref{ceadmm-sub2}) and (\ref{ceadmm-sub7}), {\tt FedGiA} only needs to calculate the gradient once for $k_0$ steps so it has relatively low computational complexity. By contrast, all other algorithms need to calculate the gradient $k_0$ times.}
 \end{remark}

\section{Numerical Experiments}\label{sec:num}
In this section, we conduct some numerical experiments to demonstrate the performance of \GIA\ in Algorithm \ref{algorithm-CEADMM}. All numerical experiments are implemented through MATLAB (R2020b) on a laptop with 32GB memory and 2.3Ghz CPU. The source codes are available at \href{https://github.com/ShenglongZhou/FedGiA}{https://github.com/ShenglongZhou/FedGiA}.

 \subsection{Testing example}
 We use Example  \ref{ex:lr} with synthetic data and Example \ref{ex:lg} with real data to conduct the numerical experiments.% Both objective functions are  gradient Lipschitz continuous.
 \begin{example}[Linear regression with non-i.i.d. data]\label{ex-linear} For this problem, local clients have their objective functions as (\ref{least-squares}). We randomly generate $d$ samples $(\ba, b)$  from three distributions: the standard normal distribution, the Student's $t$ distribution with degree $5$, and the uniform distribution in $[-5,5]$. Then we shuffle all samples and  divide them into $m$ parts $(A_i, \bb_i)$ for $m$ clients, where $A_i{=}(\ba_1^i,\ldots,\ba_{d_i}^i)^\top$ and $\bb_i{=}(b_1^i,\ldots,b_{d_i}^i)^\top$. Therefore, $d{=}d_1{+}\ldots{+} d_m$.  The data size of each part, $d_i$, is randomly chosen from  $[50,150]$. For simplicity, we fix $n{=}100$ but choose $m{\in}\{64,96,128,196,256\}$. In this regard, each client has non-i.i.d. data $(A_i, \bb_i)$. 
\end{example}
\begin{example}[Logistic regression]\label{ex-logist} For this problem,  local clients have their objective functions as (\ref{logist-loss}), where $\mu=0.001$ is fixed in the numerical experiments.  We use two real datasets described in Table \ref{tab:datasets}  to generate $(\ba, b)$. We  randomly split $d$ samples into $m$ groups corresponding to $m$ clients. 
\end{example}
\begin{table}[!th]
	\renewcommand{\arraystretch}{1.25}\addtolength{\tabcolsep}{0pt}
	\caption{Descriptions of  two real datasets.}\vspace{-3mm}
	\label{tab:datasets}
	\begin{center}
		\begin{tabular}{lllrrrr }
			\hline
Data&Datasets&	Source	&	$n$	&	$d$\\\hline
%\texttt{gis} & Gisette& libsvm& 5000& 6000\\
\texttt{qot}&	Qsar oral toxicity	&	uci	&	1024 	&	8992 	\\
 \texttt{sct}&	Santander customer transaction	&	kaggle	&	200 	&	200000 	\\
% \texttt{rtb}&	Real time bidding	&	kaggle	&	88 	&	 {1000000} 	\\
\hline
 		\end{tabular}
	\end{center}
	\vspace{-5mm}
\end{table} 
\begin{example}[Non-convex regularized logistic regression]\label{ex-logist-nonconvex} For this example, we aim at solving a non-convex problem \cite{antoniadis2011penalized, zhang2020fedpd}, where client $i\in[m]$ has the objective function  as 
\begin{eqnarray*} \label{logist-loss-reg}
 \arraycolsep=0pt\def\arraystretch{1.5}
\begin{array}{llll}
f_i(\bx){=} 
\frac{1}{d_i}\sum_{j=1}^{d_i}\left(\ln (1{+}e^{\langle\ba^i_j,\bx\rangle} ){-}b^i_j\langle\ba^i_j,\bx\rangle\right){+}\frac{\mu}{2d_i}\sum_{\ell=1}^n\frac{x_\ell^2}{1+x_\ell^2}. 
\end{array} 
\end{eqnarray*}
We fix $\mu=0.01$ in the sequel.  Samples $(\ba, b)$ are generated the same as Example \ref{ex-logist}.% where $m\in\{32,64,128\}$.% It has shown in \cite[Lemma 4]{wang2019greedy} that $f_i$ defined by (\ref{logist-loss}) is the gradient Lipschitz continuous with a constant $r_i=\lambda_{\max}(A_i^\top A_i)/4+\mu$, where $A_i$ is given similarly to Example \ref{ex-linear}.
\end{example}

\subsection{Implementations}
As mentioned in Remark \ref{remark-com}, we terminate  \GIA\ if $k\geq 10^4$ or solution
 $\bx^{\tau_{k}}$ satisfies 
 \begin{eqnarray} \label{stopping}
 \arraycolsep=1.0pt\def\arraystretch{1.5}
\begin{array}{r}
{\tt Error}:=\|\nabla f(\bx^{\tau_{k}})\|^2
 \leq {\tt tol}, 
\end{array} 
\end{eqnarray}
and initialize $\bx_i^0=\bpi_i^0=0,i\in[m]$, where ${\tt tol}=10^{-7}$ for Example \ref{ex-linear} and ${\tt tol}=(5/d)\times10^{-6}$ for Examples \ref{ex-logist} and \ref{ex-logist-nonconvex}. For every $k\in\K$, we randomly select $\alpha m$ clients to form $\ctauk$, namely, $|\ctauk|=\alpha m$ and $\alpha\in(0,1]$. Here, $\alpha=1$ means all clients are chosen.   Theorem  \ref{L-global-convergence}   suggests that $\sigma $ should be chosen to satisfy $\sigma=tr/m$, where $t$ is given in Table \ref{tab:choice-Hi}. Finally, $H_i$ is chosen as Table \ref{tab:choice-Hi}, where \GIAG\ and \GIAD\ represent \GIA\ under $H_i$ opted as a Gram and Diagonal matrix, respectively.
\begin{table}[H]
	\renewcommand{\arraystretch}{1.75}\addtolength{\tabcolsep}{1pt}
	\caption{Choices of $t$ and $H_i$, where $B_i:=A_i^\top A_i$.}\vspace{-3mm}
	\label{tab:choice-Hi}
	\begin{center}
		\begin{tabular}{lllr } \hline
  &  &	\GIAG	&	\GIAD \\ 
   &$t$ &	 $H_i$ &	 $H_i$ \\\hline
Example \ref{ex-linear}&$0.15$&	$ \frac{B_i}{d_i}$&	$ \frac{\|B_i\|}{d_i}I $	\\
Example \ref{ex-logist}&$\max\left\{0.025,\frac{4{\rm ln}(d)}{n}\right\}$	&	$\frac{B_i}{4d_i} $&	$\frac{\|B_i\|}{4d_i}I$\\
Example \ref{ex-logist-nonconvex}&$\max\left\{0.025,\frac{4{\rm ln}(d)}{n}\right\}$	&	$\frac{B_i}{4d_i}+\frac{\mu I}{d_i}$&	$\frac{\|B_i\|+4\mu}{4d_i}I$\\
\hline
 		\end{tabular}
	\end{center} \vspace{-5mm}
\end{table}  

\subsection{Numerical performance}
In this part, we conduct some simulation to demonstrate the performance of \GIA\  including  global convergence,   convergence rate, and  effect of  $k_0$ and $\ctauk$. To measure the performance, we report the following factors: $f(\bx^{\tau_{k}})$,  error  $\|\nabla f(\bx^{\tau_{k}})\|^2$, CR, and computational time (in second). We only report results of \GIA\ solving Example \ref{ex-linear}  and omit ones for  Examples \ref{ex-logist} and \ref{ex-logist-nonconvex} as the observations are similar.

\subsubsection{Global convergence with rate $O(k_0/k)$} 
We   fix  $m=128,~\alpha=0.5$, and $k_0\in\{1,5,10,15,20\}$ and present the results in Fig. \ref{fig:iterations-objective}. From the left sub-figure, as expected, all lines eventually tend to the same objective function value, well testifying Theorem \ref{L-global-convergence}.  It is clear that the bigger  $k_0>1$ (i.e., the more steps between two global aggregations), the more iterations required to reach the optimal function value.  From the right sub-figure,  the trends show that all errors vanish gradually as the rising of the number of iterations. Apparently, the bigger $k_0$, the more iterations used to converge, which well justifies Theorem \ref{complexity-thorem-gradient} that convergence rate relies on $k_0$.  
 \begin{figure}[!th]
	\centering
	\includegraphics[width=.99\linewidth]{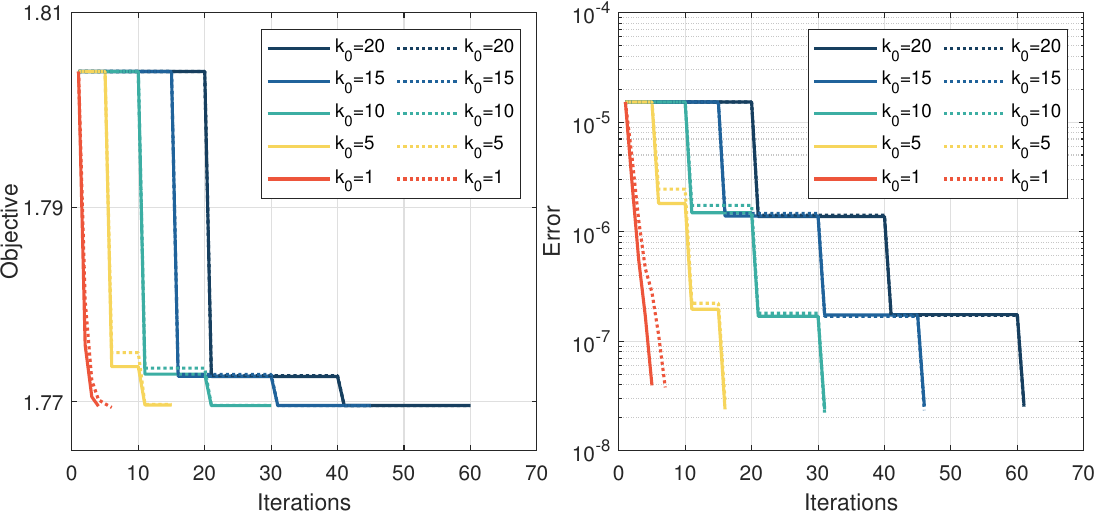}
\caption{Objective function values and errors v.s. iterations. \GIAG\ (solid lines) and \GIAD\ (dashed lines) solve Example \ref{ex-linear} with $m=128$ and $\alpha=0.5$.\label{fig:iterations-objective}}
\end{figure}

\subsubsection{Effect of $k_0$}
Next, we would like to see how the choices of $k_0$ impact the performance of \GIA. To proceed with that, for each dimension $(m,d_1,\ldots,d_m)$ of the dataset, we generate 20 instances of  Example \ref{ex-linear}. Each instance is solved by \GIA\ with fixing $\alpha=0.5$ and $k_0\in[20]$. The average results are reported in  Fig. \ref{fig:effect-k0-diff}. It can be clearly seen that CR decline first and then stabilizes at a certain level with the rising of $k_0$.  To this end, it is efficient to save the communication cost if we set a proper  $k_0$. However, it is unnecessary to set a very big value as the larger $k_0$ the longer computational time, as shown in the right sub-figure.  In general, \GIAD\ used more CR  but always ran faster than \GIAG.

\begin{figure}[!th]
%\begin{subfigure}{.33\textwidth}
%	\centering
%	\includegraphics[width=.99\linewidth]{iter-k0-eps-converted-to.pdf}
%	%\caption{Number of iterations}
%	\label{fig:k0-iter-1}
%\end{subfigure}	 
\begin{subfigure}{.24\textwidth}
	\centering
	\includegraphics[width=1.02\linewidth]{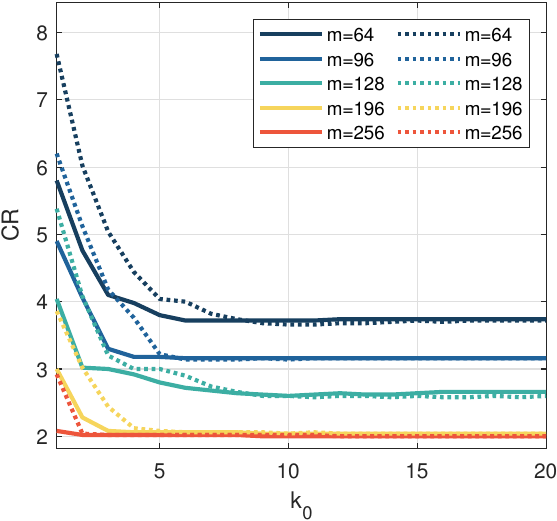}
	%\caption{CR }
	\label{fig:k0-aggr-1}
\end{subfigure}  
\begin{subfigure}{.24\textwidth}
	\centering
	\includegraphics[width=1.02\linewidth]{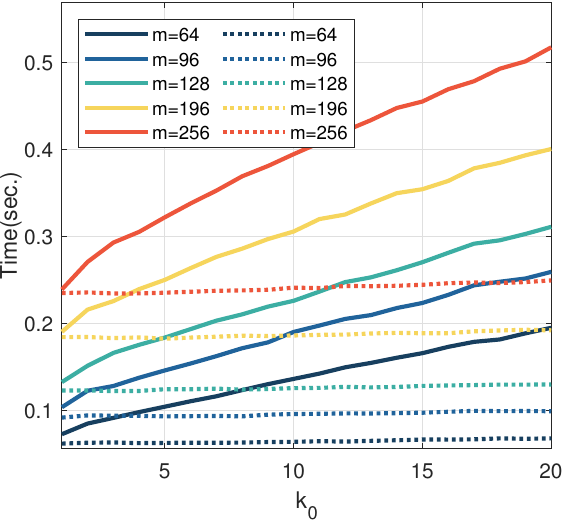}
	%\caption{Computational time}
	\label{fig:k0-time-1}
\end{subfigure} 
 \vspace{-5mm}
\caption{Effect of $k_0$ for \GIAG\ (solid lines) and \GIAD\ (dashed lines) solving Example \ref{ex-linear} with $\alpha=0.5$.\label{fig:effect-k0-diff}}\vspace{-3mm}
\end{figure}

\subsubsection{Effect of $\ctauk$} Finally, we aim to see how choices of $\ctauk$ impact the performance of \GIA\ by altering $\alpha\in(0.1,1]$. The average results are presented in Fig. \ref{fig:effect-s}. We observe that $\alpha$ would not have a big influence on CR when $k_0>5$. As expected, $\alpha$ impacts \GIAG\ greatly in terms of computational time. From the algorithmic framework, the larger $\alpha$ the more clients are selected to calculate \eqref{ceadmm-sub2}, leading to more expensive computations.  However, when $H_i$ is chosen as a diagonal matrix, computing  \eqref{ceadmm-sub2} is relatively cheap, which explains that $\alpha$ has little influence on the computational time for \GIAD. In the sequel, we fix $\alpha=0.5$ for simplicity.

 \begin{figure}[!th]
	\centering
	\begin{subfigure}{.24\textwidth}
	\centering
	\includegraphics[width=1.01\linewidth]{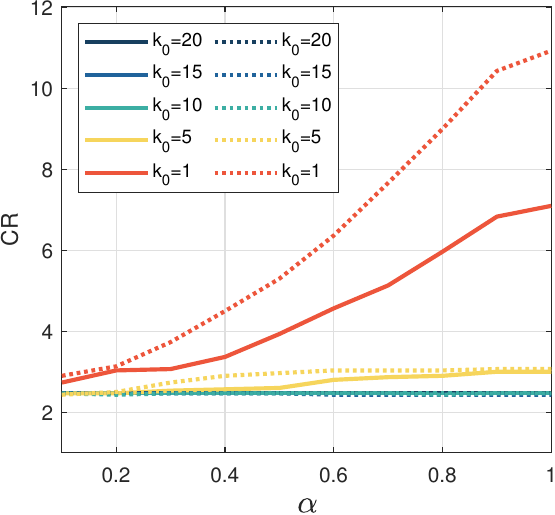}
	%\caption{CR }
	\label{fig:k0-aggr-1}
\end{subfigure}   
\begin{subfigure}{.24\textwidth}
	\centering
	\includegraphics[width=1.01\linewidth]{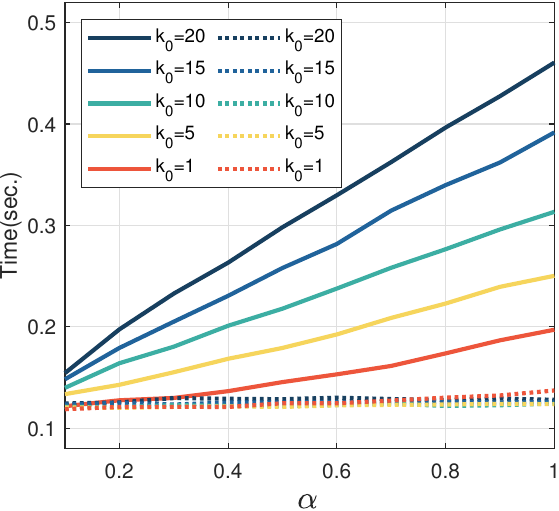}
	%\caption{Computational time}
	\label{fig:k0-time-1}
\end{subfigure}\vspace{-5mm}
\caption{Effect of $\alpha$. \GIAG\ (solid lines) and \GIAD\ (dashed lines) solve Example \ref{ex-linear} with   $m=128$.\label{fig:effect-s}}
\vspace{-5mm}
\end{figure}
 \subsection{Numerical comparison}
{In this part, we will compare our proposed method with \GD\ \cite{mcmahan2017communication},  {\tt FedPD} \cite{zhang2020fedpd}, and {\tt FedProx} \cite{li2020federatedprox}. For fair comparison, we initialize all algorithms with same starting point $\bx_i^0=0,i\in[m]$ and terminate them if condition \eqref{stopping} is satisfied or CR  are over $1000$. Since all clients participate in the training for consecutive $k_0$ steps in {\tt FedGiA}, we apply the full device participation into the other algorithms, namely, all clients are chosen to update their parameters at every step. In addition to these settings, we also set up them as follows.
 \begin{itemize}[leftmargin=10pt]
 \item For {\tt FedAvg}, we use its non-stochastic version, that is, we use full data to calculate the gradient for every client. Learning rate (i.e., the step size) is set as $\gamma_k(a):=a/{\rm log_2}(k+2)$ with $a=0.01$ for Example \ref{ex-linear} and $a=0.5d/m$ for Examples \ref{ex-logist} and \ref{ex-logist-nonconvex}.
% \item For \LocalSGD, as suggested by  \cite{Lin2020Don} using small mini-batch size to approximate the gradient for every local client,  we choose mini-batch  size $0.1d_i$ for the $i$th  client. Its learning rate is set the same as {\tt FedAvg}.
  \item For {\tt FedProx}, each client needs to solve a  regularized optimization sub-problem at every step.  The regularized penalty constant is given as $\mu=10^{-4}$. We also employ the GD method to solve the sub-problem.   The maximal number of iterations is set as  $5$, and the learning rate is set as $\gamma_k(a)$ with $a=0.001$ for Example \ref{ex-linear} and $a=0.5d/m$ for Examples \ref{ex-logist} and \ref{ex-logist-nonconvex}.  
   \item  For {\tt FedPD},  we adopt the version with oracle choice I and option I, where the maximal number of iterations for the GD method is set as $5$. This algorithm involves two parameters  $\eta$ and $\eta_1$. The former has a similar   role of $1/\sigma$ in \eqref{Def-L} and the latter is the learning rate for solving a sub-problem.  For Example \ref{ex-linear}, we set $\eta=1$ and   $\eta_1=\gamma_k(0.05)$. For Examples \ref{ex-logist} and \ref{ex-logist-nonconvex}, we set $\eta=\max\{400, d/50\}$ and  $\eta_1=\gamma_k(0.5d/m)$. Moreover, instead of conducting the global aggregation with a probability, we use the same scheme as the other algorithms, namely,  aggregating all parameters when $k\in\K$.
\end{itemize}}

\begin{figure*}
	\centering
	\includegraphics[width=1.01\linewidth]{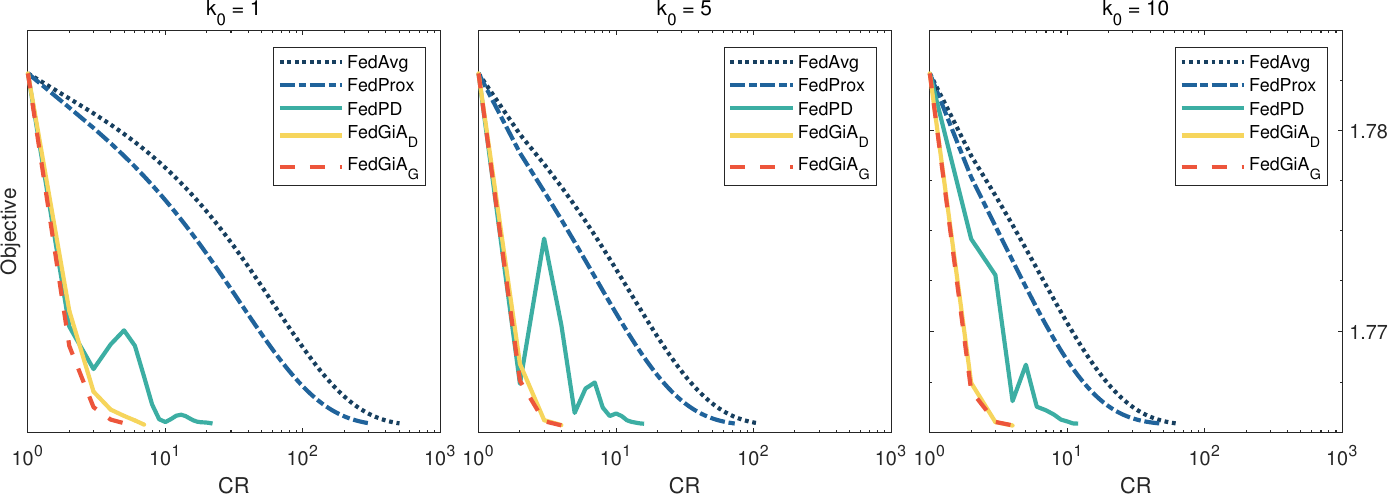}
\caption{$f(\bx^{\tau_{k}}) $  v.s. CR for Example \ref{ex-linear}.\label{fig:com-grad-cr}}
\end{figure*} 

\begin{figure*}
	\centering
	\includegraphics[width=1.01\linewidth]{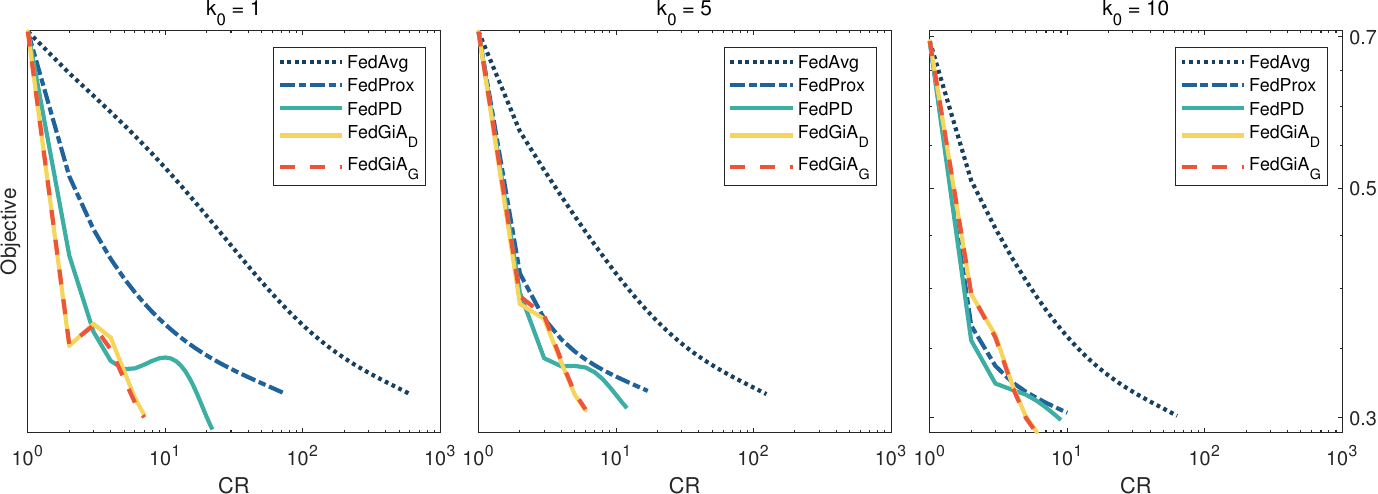}
\caption{$f(\bx^{\tau_{k}}) $  v.s. CR for Example \ref{ex-logist} with {\tt qot}.\label{fig:f-cf-ex2}}
\end{figure*} 

\begin{figure*}
	\centering
	\includegraphics[width=1.01\linewidth]{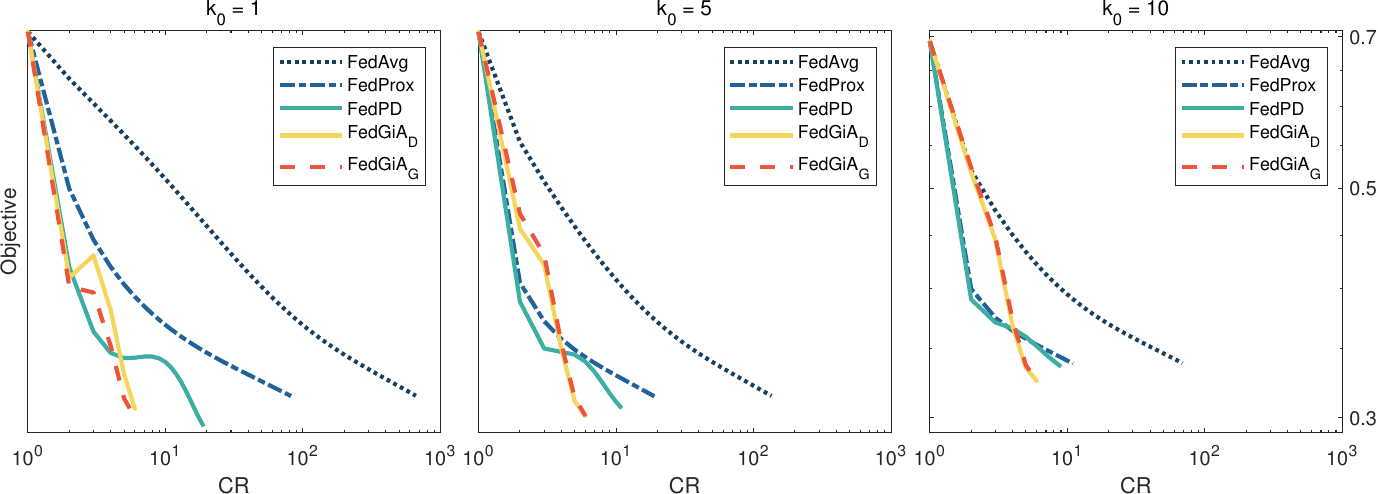}
\caption{$f(\bx^{\tau_{k}}) $  v.s. CR for Example \ref{ex-logist-nonconvex} with {\tt qot}.\label{fig:f-cf-ex3}}
\end{figure*} 

\begin{table}[!th]
	\renewcommand{\arraystretch}{1.0}\addtolength{\tabcolsep}{-3.5pt}
	\caption{Comparison for four algorithms.}\vspace{-3mm}
	\label{tab:com-linear}
	\begin{center}
		\begin{tabular}{llccccccccccc }
			\hline
			&&\multicolumn{3}{c}{$k_0=1$}&&\multicolumn{3}{c}{$k_0=5$}&&\multicolumn{3}{c}{$k_0=10$}\\\cline{3-5}\cline{7-9}\cline{11-13}
Algs.	&	&	Obj.	&	CR	&	Time	&	&	Obj.	&	CR	&	Time	&	&	Obj.	&	CR	&	Time	\\\hline
&&\multicolumn{11}{c}{Example \ref{ex-linear}}\\\cline{3-13}
{\tt FedAvg}	&	&	1.744 	&	515.3 	&	1.18 	&	&	1.744 	&	112.0 	&	0.70 	&	&	1.744 	&	63.7 	&	0.64 	\\
{\tt FedProx}	&	&	1.744 	&	315.7 	&	1.64 	&	&	1.744 	&	74.3 	&	1.41 	&	&	1.744 	&	49.9 	&	1.83 	\\
{\tt FedPD}	&	&	1.744 	&	21.9 	&	0.15 	&	&	1.744 	&	15.1 	&	0.29 	&	&	1.744 	&	11.2 	&	0.42 	\\
{\tt FedGiA$_{\tt D}$}	&	&	1.744 	&	6.1 	&	0.11 	&	&	1.744 	&	3.0 	&	0.11 	&	&	1.744 	&	3.0 	&	0.11 	\\
{\tt FedGiA$_{\tt G}$}	&	&	1.744 	&	4.5 	&	0.13 	&	&	1.744 	&	3.0 	&	0.16 	&	&	1.744 	&	3.0 	&	0.19 	\\
\hline
&&\multicolumn{11}{c}{Example \ref{ex-logist} with {\tt qot}}\\\cline{3-13}
{\tt FedAvg}	&	&	0.301 	&	597.4 	&	6.64 	&	&	0.301 	&	120.9 	&	3.83 	&	&	0.301 	&	61.2 	&	3.53 	\\
{\tt FedProx}	&	&	0.301 	&	71.7 	&	2.01 	&	&	0.301 	&	15.8 	&	1.81 	&	&	0.301 	&	8.7 	&	1.93 	\\
{\tt FedPD}	&	&	0.288 	&	20.5 	&	0.68 	&	&	0.292 	&	10.8 	&	1.31 	&	&	0.299 	&	7.6 	&	1.74 	\\
{\tt FedGiA$_{\tt D}$}	&	&	0.285 	&	5.7 	&	0.33 	&	&	0.285 	&	5.1 	&	0.34 	&	&	0.285 	&	5.0 	&	0.35 	\\
{\tt FedGiA$_{\tt G}$}	&	&	0.287 	&	5.2 	&	0.39 	&	&	0.285 	&	5.0 	&	0.60 	&	&	0.285 	&	4.9 	&	0.79 	\\
\hline
&&\multicolumn{11}{c}{Example \ref{ex-logist} with {\tt sct}}\\\cline{3-13}
{\tt FedAvg}	&	&	0.332 	&	190.4 	&	8.51 	&	&	0.332 	&	38.7 	&	5.02 	&	&	0.332 	&	19.7 	&	4.65 	\\
{\tt FedProx}	&	&	0.332 	&	21.0 	&	1.99 	&	&	0.332 	&	7.0 	&	2.51 	&	&	0.332 	&	5.0 	&	3.40 	\\
{\tt FedPD}	&	&	0.331 	&	4.0 	&	0.46 	&	&	0.332 	&	4.0 	&	1.49 	&	&	0.331 	&	5.0 	&	3.40 	\\
{\tt FedGiA$_{\tt D}$}	&	&	0.329 	&	5.0 	&	0.41 	&	&	0.328 	&	4.0 	&	0.39 	&	&	0.328 	&	4.0 	&	0.39 	\\
{\tt FedGiA$_{\tt G}$}	&	&	0.329 	&	4.4 	&	0.46 	&	&	0.328 	&	4.0 	&	0.62 	&	&	0.328 	&	4.0 	&	0.77 	\\
\hline
&&\multicolumn{11}{c}{Example \ref{ex-logist-nonconvex} with {\tt qot}}\\\cline{3-13}
{\tt FedAvg}	&	&	0.331 	&	682.7 	&	7.90 	&	&	0.331 	&	138.3 	&	4.53 	&	&	0.331 	&	70.3 	&	4.15 	\\
{\tt FedProx}	&	&	0.331 	&	83.7 	&	2.43 	&	&	0.331 	&	19.1 	&	2.27 	&	&	0.330 	&	11.0 	&	2.51 	\\
{\tt FedPD}	&	&	0.322 	&	18.5 	&	0.66 	&	&	0.324 	&	10.0 	&	1.29 	&	&	0.330 	&	8.0 	&	1.89 	\\
{\tt FedGiA$_{\tt D}$}	&	&	0.318 	&	5.7 	&	0.34 	&	&	0.322 	&	5.1 	&	0.35 	&	&	0.322 	&	4.9 	&	0.36 	\\
{\tt FedGiA$_{\tt G}$}	&	&	0.320 	&	5.2 	&	0.40 	&	&	0.321 	&	5.0 	&	0.61 	&	&	0.321 	&	4.9 	&	0.81 	\\

\hline
&&\multicolumn{11}{c}{Example \ref{ex-logist-nonconvex} with {\tt sct}}\\\cline{3-13}
{\tt FedAvg}	&	&	0.327 	&	200.0 	&	8.66 	&	&	0.327 	&	40.3 	&	4.99 	&	&	0.327 	&	20.3 	&	4.59 	\\
{\tt FedProx}	&	&	0.327 	&	22.2 	&	2.03 	&	&	0.327 	&	7.2 	&	2.50 	&	&	0.326 	&	6.0 	&	3.98 	\\
{\tt FedPD}	&	&	0.327 	&	4.5 	&	0.44 	&	&	0.327 	&	4.2 	&	1.50 	&	&	0.327 	&	5.0 	&	3.34 	\\
{\tt FedGiA$_{\tt D}$}	&	&	0.326 	&	4.2 	&	0.37 	&	&	0.324 	&	4.0 	&	0.37 	&	&	0.324 	&	4.0 	&	0.37 	\\
{\tt FedGiA$_{\tt G}$}	&	&	0.325 	&	4.1 	&	0.43 	&	&	0.324 	&	4.0 	&	0.58 	&	&	0.324 	&	4.0 	&	0.77 	\\

\hline
		\end{tabular}
	\end{center}
\end{table}
{
\subsubsection{Solving Example \ref{ex-linear}} For simplicity, we fix $m=128$ and $n=100$. From Fig. \ref{fig:com-grad-cr}, the objective function values for all methods eventually tend to the same one. Basically, the larger $k_0$ the fewer CR used to converge. Moreover, \GIAG\ and \GIAD\ outperform the others as they use the fewest CR.  We then run 20 independent trials and report the average results in Table \ref{tab:com-linear}. Clearly, there is no difference on the objective function values among these algorithms. However, in terms of using CR, \GIAG\ performs the best, followed by \GIAD\ and {\tt FedPD}. For the computational speed,  \GIAD\ always runs the fastest.}

{
\subsubsection{Solving Examples \ref{ex-logist} and \ref{ex-logist-nonconvex}} Again, we fix $m=128$ for simplicity. From Fig. \ref{fig:f-cf-ex2} and Fig. \ref{fig:f-cf-ex3}, when $k_0=5$ and $10$, although {\tt FedProx} and {\tt FedPD} can decline very quickly at the beginning, they still need higher CR than \GIA.  The objective function values of {\tt FedAvg} always decrease the slowest. Then we report the average results over 20 independent trials in Table \ref{tab:com-linear}, from which we can conclude several conclusions. First,  \GIA\ obtains the smallest objective function values and consumes the lowest CR almost for all scenarios, followed by {\tt FedPD}. Evidently, {\tt FedAvg} comes the last. In addition, both \GIAD\ and \GIAG\ run the fastest. By contrast, {\tt FedAvg} takes the longest time since it uses the highest CR.}

\section{Conclusion}
We developed a new FL algorithm and managed to address three critical issues in FL, including saving communication resources, reducing computational complexity, and establishing convergence property under mild assumptions. These advantages hint that the proposed algorithm might be practical to deal with many real applications, such as mobile edge computing \cite{mao2017survey,mach2017mobile,liu2022distributed}, over-the-air computation \cite{zhu2018mimo,yang2020federated}, vehicular
communications \cite{samarakoon2019distributed}, unmanned aerial vehicle online path control \cite{shiri2020communication}, immersive virtual reality video streaming \cite{chen2019federated, zheng2020mec},  and so forth.  Moreover, we feel that the algorithmic schemes and techniques used to build the convergence theory could be also valid for tackling decentralized FL  \cite{elgabli2020fgadmm,ye2022decentralized}. We leave these for future research.

%In addition, in our algorithms, we made use of full local data for every client. However,  the stochastic gradient descent methods randomly pick a small portion of samples from local data to concrete the gradient. Therefore, if similar random sampling is used, what the theoretical and numerical performance for our algorithms is. Moreover, are our algorithms effective to deal with the Byzantine attacks \cite{vempaty2013distributed,blanchard2017machine}?

\bibliographystyle{IEEEtran}
\bibliography{ref}

%\newpage
\appendices
\section{Some Basics}
For notational simplicity, hereafter, we denote
\begin{eqnarray*} 
% \label{decreasing-property-0}   
 \arraycolsep=1.0pt\def\arraystretch{1.5}
\qquad \begin{array}{llllll}
\triangle\bx_i^{k+1}&:=&\bx_i^{k+1}-\bx_i^{k},~~
&\triangle\bpi_i^{k+1}&:=&\bpi_i^{k+1}-\bpi_i^{k},\\
\triangle\bx^{\tau_{k+1}}&:=&\bx^{\tau_{k+1}}-\bx^{\tau_{k}},~~&\triangle \blx_i^{k+1}&:=&\bx_i^{k+1}-\bx^{\tau_{k+1}},\\
 \bg_i^{k}&:=&\frac{1}{m}\nabla f_i(\bx_i^{k}),~~&\overline\bg_i^{k+1}&:=&\frac{1}{m}\nabla f_i(\bx^{\tau_{k+1}}).
    \end{array}  
 \end{eqnarray*} 
For any  vectors $\ba,\bb, \ba_i$, matrix $H\succeq0$, and $t>0$, we have
\begin{eqnarray} \label{two-vecs}
 \arraycolsep=1.4pt\def\arraystretch{1.5}
 \begin{array}{rcl}
 -  \|\bb\|^2 &=&2\langle\ba,\bb\rangle +   \|\ba\|^2-  \|\ba+\bb\|^2,\\
\|\ba+\bb\|^2 &\leq& (1+t)\|\ba\|^2+(1+1/t)\|\bb\|^2,\\
\|\sum_{i=1}^m\ba_i\|^2 &\leq& m\sum_{i=1}^m\|\ba_i\|^2,\\
2\langle H\ba,\bb\rangle &\leq& t\|\ba\|_H^2+(1/t) \|\bb\|_H^2.
 \end{array}
\end{eqnarray}
By the Mean Value Theorem, the gradient Lipschitz continuity indicates that for any $\bx, \bz$ and $\bw\in\{\bx, \bz\}$,
\begin{eqnarray}  \label{H-Lip-continuity-fxy}
 \arraycolsep=1.4pt\def\arraystretch{1.5}
\qquad \begin{array}{llll}
&&f(\bx) - f(\bz  ) -\langle \nabla  f(\bw  ), \bx-\bz   \rangle\\
%&=&  \int_0^1 d f(\bz+t(\bx-\bz))-\langle \nabla  f(\bw  ), \bx-\bz   \rangle\\
&=&  \int_0^1 \langle \nabla  f(\bz+t(\bx-\bz)) - \nabla  f(\bw  ), \bx-\bz   \rangle dt  \\
&\leq&  \int_0^1 r\|\bz+t (\bx-\bz)  -  \bw \|\| \bx-\bz   \| dt\\
&=&  \frac{r}{2}\| \bx-\bz   \|^2.
 \end{array}
\end{eqnarray} 
{\begin{definition}[Desingularizing Function]\label{def-Desingularizing}  A function $ \varphi:[0,\eta)\to (0,+\infty)$ satisfying the following conditions is called a desingularizing function:
\begin{itemize}
\item[i)] $ \varphi$ is concave and continuously differentiable on $(0, \eta)$;
\item[ii)]$ \varphi$ is continuous at 0, $\varphi(0) = 0$, and
\item[iii)] $ \varphi'(x)> 0, \forall x \in (0, \eta)$.
\end{itemize}
Let $\Phi_{\eta}$ be the set of desingularizing functions defined on $[0,\eta)$. 
\end{definition}
For an extended-real-valued function $f:\R^n\to[-\infty,+\infty]$, the domain is defined as ${\rm dom} f = \{\bx : f(\bx) < \infty\}$. A function is called proper if it never reaches $-\infty$ and its domain is nonempty, and is called closed if it is lower semicontinuous. Denote $\partial f $  the (limiting) subdifferential of a proper function $f$ \cite{rockafellar2009variational}. If $f$ is continuously differentiable, then subdifferential $\partial f $   reduces to the gradient of $f$, namely $\partial f=\{\nabla f\}.$ Function ${\rm dist}$ is defined as
$${\rm dist}(\bx,\Omega):=\inf_{\bz\in\Omega}\|\bx-\bz\|.$$ Based on these definitions,   we introduce the KL functions \cite{bolte2014proximal}.
\begin{definition}[KL Property]\label{def-KL-func}  A proper closed function $f:\R^n\to(-\infty,+\infty]$ is said to have the Kurdyka-Lojasiewicz (KL) property at $\bx^* \in
{\rm dom} (\partial f) :=\{\bx \in\R^n: \partial f (\bx) \neq \emptyset \}$ if there exists $\eta \in (0,+\infty]$, a neighborhood
$U$ of $\bx^*$, and a desingularizing function $\varphi \in \Phi_{\eta}$, such that for all
$$\bx \in U \cap \{\bx \in\R^n: f (\bx^*) < f(\bx) < f(\bx^*) + \eta\},$$
the following inequality holds:
$$\varphi'(f (\bx) - f (\bx^*)){\rm dist}(0, \partial f (\bx)) > 1.$$
If $f$ satisfies the KL property at each point of ${\rm dom} \partial f $, then $f$ is called a KL function.
\end{definition}}

\section{Proofs of all theorems}

\subsection{Key lemmas}
\begin{lemma}\label{basic-observations}  Let $\{\Z^{k} \}$ be the sequence generated by Algorithm \ref{algorithm-CEADMM} with $H_i=\Theta_i,i\in[m]$. The following results hold under Assumption \ref{ass-fi}.
\begin{itemize}
%\item[b)]$\forall~k\geq 0, \forall~ i\in[m]$,
%\begin{eqnarray}
% \label{opt-con-xk1-0}
%\arraycolsep=1.4pt\def\arraystretch{1.5}
%\begin{array}{rcll}
%\bpi^{k}_i &=&   \bpi_i^{k-1} +\sigma(\bx_i^{k}-\bx^{\tau_{k}}), \\
%\bz^{k}_i &=&   \bx_i^{k}+\bpi^{k}_i/\sigma.  
%\end{array} \end{eqnarray} 
\item[a)]$\forall~k\in\K$, 
\begin{eqnarray}
 \label{opt-con-xk1-1}
\begin{array}{rcll}
\sum_{i=1}^m  ( \frac{\bpi_i^{k}}{\sigma }   +\bx_i^{k}-\bx^{\tau_{k+1}})=0.
\end{array} \end{eqnarray} 
\item[b)]$\forall~k\geq 0, \forall~i\in[m]$,
\begin{eqnarray}
 \label{opt-con-xk1-2}
\begin{array}{rcll} 
\overline\bg_i^{k+1}  +  \bpi_i^{k+1} + \frac{1}{m}    H_i    \triangle\blx_i^{k+1}=0.
\end{array} \end{eqnarray} 
\item[c)] $\forall~ k\geq 0, \forall~i\in[m]$,  
\begin{eqnarray}
 \label{opt-con-xk1-3}
\begin{array}{rcll}
 \|\triangle \bpi_i^{k+1} \|^2  \leq       \frac{6r_i^2}{m^2}  \|\triangle\bx^{\tau_{k+1}}\|^2+\frac{3r_i^2}{m^2}    \|\triangle \bx_i^{k+1}\|^2.
\end{array} \end{eqnarray} 
\end{itemize}
\end{lemma}  
\begin{proof} 
a) For any $i\notin\ctauk$,  we have from \eqref{ceadmm-sub4} that
 \begin{eqnarray}\label{z-pi-x}
\arraycolsep=1.4pt\def\arraystretch{1.5}
\begin{array}{lcl}
\bz^{k+1}_i = \frac{\bpi_i^{k+1}}{\sigma }  +\bx_i^{k+1}.
\end{array} \end{eqnarray} 
For any $i\notin\ctauk$, it follows from \eqref{ceadmm-sub6}-\eqref{ceadmm-sub8} that the above relation is still valid.  Hence,   we have \eqref{z-pi-x} for any $i\in[m]$ and for any $k\geq0$. As a result,  for any $k\in\K$,  
 \begin{eqnarray*}
\arraycolsep=1.4pt\def\arraystretch{1.5}
~~\begin{array}{lcl}
\sum_{i=1}^m  (  \frac{\bpi_i^{k}}{\sigma }  +\bx_i^{k}-\bx^{\tau_{k+1}}) \overset{\eqref{z-pi-x}}{=} \sum_{i=1}^m ( \bz_i^{k} -\bx^{\tau_{k+1}})   \overset{\eqref{ceadmm-sub1}}{=} 0.
\end{array} \end{eqnarray*} 
b) For $i\in\ctauk$,  solution $\bx_i^{k+1}$ in \eqref{ceadmm-sub2}  satisfies  \eqref{iceadmm-sub2-0}, thereby contributing to, 
 \begin{eqnarray}
 \label{opt-con-xik1-10}
 \arraycolsep=1.5pt\def\arraystretch{1.5}
\begin{array}{lcl}
 0 &=&  \overline\bg_i^{k+1} + \bpi_i^{k} + (\frac{1}{m}    H_i  +\sigma I)   \triangle\blx_i^{k+1}\\
 &\overset{\eqref{ceadmm-sub3}}{=}&  \overline\bg_i^{k+1}  +  \bpi_i^{k+1} + \frac{1}{m}    H_i \triangle\blx_i^{k+1}.
   \end{array}
  \end{eqnarray} 
  For any $i\notin\ctauk$, the second equation in \eqref{opt-con-xik1-10} is still valid due to $\bpi_i^{k+1} =- \blg_i^{k+1}  $ and $\triangle\blx_i^{k+1}=0$ from \eqref{ceadmm-sub6}-\eqref{ceadmm-sub8}.   Hence, it is true for any $i\in[m]$ and any $k\in\K$.

c) It follows from  \eqref{opt-con-xk1-2} and $r_iI\succeq H_i=\Theta_i\succeq 0$ that 
 \begin{eqnarray}\label{gap-pi-k}
 \arraycolsep=1.5pt\def\arraystretch{1.5}
 \begin{array}{lcl}
%  && \|\triangle \bpi_i^{k+1} \|^2\\
%  &=&\|\overline\bg_i^{k+1}-\overline\bg_i^{k}  + \frac{1}{m}    H_i (\triangle\bx_i^{k+1}-\triangle\bx_i^{\tau_{k+1}}) \|^2\\
%&\overset{\eqref{two-vecs} }{\leq}&   \frac{3r_i^2}{m^2}    \|\triangle \bx_i^{k+1}\|^2+ \frac{3r_i^2}{m^2}  \|\triangle\bx^{\tau_{k+1}}\|^2+ 3\| \overline\bg_i^{k+1}  -\overline\bg_i^{k}\|^2 \\
% &\overset{\eqref{Lip-r} }{\leq}&  \frac{3r_i^2}{m^2}    \|\triangle \bx_i^{k+1}\|^2+   \frac{6r_i^2}{m^2}  \|\triangle\bx^{\tau_{k+1}}\|^2,
&& \|\triangle \bpi_i^{k+1} \|\\
  &=&\|\overline\bg_i^{k+1}-\overline\bg_i^{k}  + \frac{1}{m}    H_i (\triangle\bx_i^{k+1}-\triangle\bx_i^{\tau_{k+1}}) \|\\
&{\leq}&  \| \overline\bg_i^{k+1}  -\overline\bg_i^{k}\|+ \frac{r_i}{m}    \|\triangle \bx_i^{k+1}\|+ \frac{r_i}{m}  \|\triangle\bx^{\tau_{k+1}}\| \\
 &\overset{\eqref{Lip-r} }{\leq}&  \frac{2r_i}{m}  \|\triangle\bx^{\tau_{k+1}}\|+\frac{r_i}{m}    \|\triangle \bx_i^{k+1}\|,
   \end{array}
  \end{eqnarray} 
  which by \eqref{two-vecs} derives the result.
\end{proof}
\subsection{Proof of Lemma \ref{lemma-decreasing-0} }
\begin{proof} i) It is easy to see
  \begin{eqnarray} 
\label{ri-m-sigma}  
 \arraycolsep=1.4pt\def\arraystretch{1.5}
  \begin{array}{llllll}
 \frac{r_i}{m} \leq  \frac{r}{m}\leq \frac{\sigma}{6},
    \end{array} 
 \end{eqnarray} 
 which recalling \eqref{opt-con-xk1-3} gives rise to
 \begin{eqnarray} \label{opt-con-xk1-3-simple}
   \arraycolsep=1.4pt\def\arraystretch{1.5}
  \begin{array}{lcl}
 \|\triangle \bpi_i^{k+1}  \|^2 \leq  \frac{\sigma^2}{6 } \|\triangle\bx^{\tau_{k+1}}\|^2+\frac{ \sigma^2}{12}    \|\triangle \bx_i^{k+1}\|^2. 
    \end{array} 
 \end{eqnarray}
We note that gap  $(\L^{k+1}-\L^{k})$ can be decomposed as 
  \begin{eqnarray} 
\label{three-cases}  
 \arraycolsep=1.4pt\def\arraystretch{1.5}
 \qquad   \begin{array}{llllll}
 {\L}(\Z^{k+1}) -{\L} (\Z^{k}) =: e_1^k+e_2^k+e_3^k,
    \end{array} 
 \end{eqnarray} 
 with
   \begin{eqnarray}  \label{three-cases-sub}  
 \arraycolsep=1.4pt\def\arraystretch{1.5}
 \begin{array}{llllll}
e_1^k&:=&\L(\bx^{\tau_{k+1}},X^{k},\Pi^{k})-\L(\Z^{k}),\\
 e_2^k&:=& \L(\bx^{\tau_{k+1}},X^{k+1},\Pi^{k})-\L(\bx^{\tau_{k+1}},X^{k},\Pi^{k}), \\
 e_3^k&:=& \L(\Z^{k+1})-\L(\bx^{\tau_{k+1}},X^{k+1},\Pi^{k}) .
    \end{array} 
 \end{eqnarray} 
\underline{Estimating  $e_1^k$.} If  $k\notin\K$, then  $\bx^{\tau_{k+1}}=\bx^{\tau_k}$, yielding
   \begin{eqnarray*} 
 \arraycolsep=1.4pt\def\arraystretch{1.5}
 \begin{array}{llllll}
e_1^k&=&0%\sum_{i=1}^m(L(\by_i^{k+1},\bx_i^{k},\bpi_i^{k})-L(\by_i^{k},\bx_i^{k},\bpi_i^{k})) 
= - \frac{ \sigma m}{2} \|\triangle\bx^{\tau_{k+1}} \|^2. 
    \end{array} 
 \end{eqnarray*} 
For  $k\in\K$, multiplying both sides of the first equation in \eqref{opt-con-xk1-1} by $\triangle\bx^{\tau_{k+1}}$ yields
% \begin{eqnarray*} 
%    \arraycolsep=1.4pt\def\arraystretch{1.5}
%   \begin{array}{lll}
%&& \sum_{i=1}^{m} \langle \triangle\bx^{\tau_{k+1}},\bpi_i^k  \rangle\\
%&=& \sum_{i=1}^{m} \langle \triangle\bx^{\tau_{k+1}},  \sigma  (\bx^{\tau_k}-\bx_i^k)\rangle, 
%\end{array}
%  \end{eqnarray*} 
%  which is equivalent to 
   \begin{eqnarray}
 \label{opt-con-xk2}
    \arraycolsep=1.4pt\def\arraystretch{1.5}
   \begin{array}{lll} \sum_{i=1}^{m} \langle \triangle \bx^{\tau_{k+1}},\bpi_i^k  \rangle = \sum_{i=1}^{m} \langle \triangle \bx^{\tau_{k+1}},  \sigma  (\bx^{\tau_{k+1}}-\bx_i^k)\rangle. 
\end{array}
  \end{eqnarray} 
The fact allows  us to derive that
\begin{eqnarray*} 
   \arraycolsep=1.4pt\def\arraystretch{1.5}
\begin{array}{lcl}
  e_1^k 
   &\overset{\eqref{Def-L}}{=}& \sum_{i=1}^{m} ( L(\bx^{\tau_{k+1}},\bx_i^{k},\bpi_i^{k})  - L(\bx^{\tau_{k}},\bx_i^{k},\bpi_i^{k}))\\
    &\overset{\eqref{Def-L}}{=}&    \sum_{i=1}^{m} ( \langle \triangle\bx^{\tau_{k+1}}, -\bpi_i^k\rangle \\
    & +&   \frac{\sigma }{2}\|\bx_i^k-\bx^{\tau_{k+1}}\|^2 -\frac{\sigma }{2}\|\bx_i^k- \bx^{\tau_{k}}\|^2) \\
&\overset{\eqref{opt-con-xk2}}{=}& \sum_{i=1}^{m} (  \langle  \triangle\bx^{\tau_{k+1}},  \sigma (\bx_i^k-\bx^{\tau_{k+1}}) \rangle\\
&+&   \frac{\sigma }{2}\|\bx_i^k-\bx^{\tau_{k+1}}\|^2 -\frac{\sigma }{2}\|\bx_i^k- \bx^{\tau_{k}}\|^2) \\
&\overset{\eqref{two-vecs}}{=}&   -\frac{\sigma}{2}\sum_{i=1}^{m} \| \triangle\bx^{\tau_{k+1}}\|^2= -\frac{\sigma m}{2} \| \triangle\bx^{\tau_{k+1}}\|^2.
    \end{array} 
 \end{eqnarray*} 
 Overall, for both scenarios, we obtained
\begin{eqnarray} \label{two-cases}  
   \arraycolsep=1.4pt\def\arraystretch{1.5}
\begin{array}{lcl}
  e_1^k = -\frac{\sigma m}{2} \| \triangle\bx^{\tau_{k+1}}\|^2.
    \end{array} 
 \end{eqnarray}  
\underline{Estimating  $e_2^k$.} We denote
\begin{eqnarray}  \label{def-pi}
   \arraycolsep=0.5pt\def\arraystretch{1.5}
  \begin{array}{lcl}
p_{i}^k&:=&  L(\bx^{\tau_{k+1}},\bx_i^{k+1},\bpi_i^{k})  - L(\bx^{\tau_{k+1}},\bx_i^{k},\bpi_i^{k})\\
&\overset{\eqref{Def-L}}{=}&  \frac{1}{m}  f_i(\bx_i^{k+1})- \frac{1}{m}  f_i(\bx_i^k)   +  \langle \triangle \bx_i^{k+1} , \bpi_i^k\rangle \\
     &+ &    \frac{\sigma }{2}\|\triangle\blx_i^{k+1} \|^2 -\frac{\sigma }{2}\|\bx_i^{k}-\bx^{\tau_{k+1}}\|^2. 
    \end{array} 
 \end{eqnarray}
 We will consider  two cases: $i\notin\ctauk$ and  $i\in\ctauk$. For $i\notin\ctauk$, if $k\in\K$, then  $\bx_i^{k+1}\equiv\bx^{\tau_{k+1}}$  (namely $\triangle \blx_i^{k+1}=0$) suffices to
 \begin{eqnarray*}  
   \arraycolsep=0.5pt\def\arraystretch{1.5}
  \begin{array}{lcl}
p_{i}^k%&=&  L(\bx^{\tau_{k+1}},\bx_i^{k+1},\bpi_i^{k})  - L(\bx^{\tau_{k+1}},\bx_i^{k},\bpi_i^{k})\\
%&=&  L(\bx_i^{k+1},\bx_i^{k+1},\bpi_i^{k})  - L(\bx_i^{k+1},\bx_i^{k},\bpi_i^{k})\\
&\overset{\eqref{def-pi}}{=}&  \frac{1}{m}  f_i(\bx_i^{k+1}) - \frac{1}{m}  f_i(\bx_i^k)   + \langle \triangle \bx_i^{k+1} , \bpi_i^k\rangle -    \frac{\sigma }{2}\|\triangle \bx_i^{k+1} \|^2\\
&\overset{\eqref{H-Lip-continuity-fxy}}{\leq}&  \langle \triangle \bx_i^{k+1} , \bg_i^{k+1}+\bpi_i^k\rangle + \frac{r_i}{2m} \|\triangle \bx_i^{k+1} \|^2 -    \frac{\sigma }{2}\|\triangle \bx_i^{k+1} \|^2\\
&\overset{\eqref{ri-m-sigma}}{\leq}&  \langle \triangle \bx_i^{k+1} , \blg_i^{k+1}+\bpi_i^k\rangle - \frac{5\sigma }{12} \|\triangle \bx_i^{k+1} \|^2\\
& \overset{\eqref{ceadmm-sub7}  }{=}&  \langle \triangle \bx_i^{k+1} , -\triangle\bpi_i^{k+1}\rangle -  \frac{5\sigma }{12} \|\triangle \bx_i^{k+1} \|^2\\
&  \overset{\eqref{two-vecs}}{\leq}&  \frac{r_i}{m} \| \triangle \bx_i^{k+1} \|^2+ \frac{m}{4r_i} \|\triangle\bpi_i^{k+1}\|^2 -  \frac{5\sigma }{12}\|\triangle \bx_i^{k+1} \|^2\\
&  \overset{\eqref{opt-con-xk1-3}}{\leq}&     \frac{3r_i}{2m}  \|\triangle\bx^{\tau_{k+1}}\|^2 +    \frac{7r_i}{4m}  \|\triangle\bx_i^{{k+1}}\|^2 -  \frac{5\sigma }{12}\|\triangle \bx_i^{k+1} \|^2\\
&  \overset{\eqref{ri-m-sigma}}{\leq}&     \frac{\sigma}{4}  \|\triangle\bx^{\tau_{k+1}}\|^2-  \frac{\sigma }{8}\|\triangle \bx_i^{k+1} \|^2.
    \end{array} 
 \end{eqnarray*}
If $k\notin\K$, then   \eqref{ceadmm-sub6} indicates $\bx_i^{k+1}=\bx^{\tau_{k+1}}=\bx^{\tau_{k}}=\bx^{k}_i$. This immediately results in  $p_{i}^k\overset{}{=} 0 $ from \eqref{def-pi} and thus the above condition also holds.  
 Therefore, for any $i\notin\ctauk$, we showed 
 \begin{eqnarray*}  
   \arraycolsep=1.5pt\def\arraystretch{1.5}
  \begin{array}{lcl}
p_{i}^k &\leq&  \frac{\sigma}{4}  \|\triangle\bx^{\tau_{k+1}}\|^2-  \frac{\sigma }{8}\|\triangle \bx_i^{k+1} \|^2.
    \end{array} 
 \end{eqnarray*} 
For any $i\in\ctauk$, direct calculation yields that 
\begin{eqnarray} \label{i-c-k1}
   \arraycolsep=1.5pt\def\arraystretch{1.5}
  \begin{array}{lcl}
 &&\langle  \triangle \bx_i^{k+1} , \bg_i^{k+1} + \bpi_i^k+\sigma \triangle\blx_i^{k+1}\rangle \\
 &=&     \langle  \triangle \bx_i^{k+1} , \bg_i^{k+1}-\blg_i^{k+1}+\blg_i^{k+1} + \bpi_i^k+\sigma \triangle\blx_i^{k+1}\rangle \\
   &\overset{ \eqref{opt-con-xik1-10}}{=}&  \langle \triangle\bx_i^{k+1}, \bg_i^{k+1}-\blg_i^{k+1}-  \frac{1}{m}H_i   \triangle\blx_i^{k+1}  \rangle\\
    &=&  \langle \triangle \bx_i^{k+1}, \bg_i^{k+1}-\blg_i^{k+1}-\frac{1}{m}H_i   \triangle\blx_i^{k+1} 
    \rangle\\
     &\overset{\eqref{two-vecs}}{\leq}&  \frac{r_i}{4m} \|\triangle \bx_i^{k+1}\|^2+    \frac{m}{ r_i} \|\bg_i^{k+1}-\blg_i^{k+1}-\frac{1}{m}H_i   \triangle\blx_i^{k+1}\|^2\\
       &\leq&  \frac{r_i}{4m} \|\triangle \bx_i^{k+1}\|^2+    \frac{4r_i}{m} \|\triangle\blx_i^{k+1}\|^2,
    \end{array} 
 \end{eqnarray}
 where the last inequality is from $r_iI\succeq H_i=\Theta_i\succeq 0$ and the gradient Lipschitz continuity of $f_i$.  Moreover, it follows from \eqref{two-vecs} that
 \begin{eqnarray*} 
   \arraycolsep=1.5pt\def\arraystretch{1.5}
  \begin{array}{lcl}
 && \frac{\sigma }{2}\|\triangle\blx_i^{k+1} \|^2 -\frac{\sigma }{2}\|\bx_i^{k}-\bx^{\tau_{k+1}}\|^2\\
  &=& \langle \triangle\bx_i^{k+1},   \sigma \triangle\blx_i^{k+1}\rangle-  \frac{\sigma}{2} \| \triangle\bx_i^{k+1}\|^2.
    \end{array} 
 \end{eqnarray*}
 Using the above two facts derives 
\begin{eqnarray*} 
   \arraycolsep=1.5pt\def\arraystretch{1.5}
  \begin{array}{lcl}
p_i^k%&:=&     L(\bx^{\tau_{k+1}},\bx_i^{k+1},\bpi_i^{k})  - L(\bx^{\tau_{k+1}},\bx_i^{k},\bpi_i^{k})   \\
     &\overset{\eqref{def-pi}}{=}&   \frac{1}{m}  f_i(\bx_i^{k+1})- \frac{1}{m}  f_i(\bx_i^k)   +  \langle \triangle \bx_i^{k+1} , \bpi_i^k\rangle \\
     &+ &       \langle \triangle\bx_i^{k+1},   \sigma \triangle\blx_i^{k+1}\rangle-  \frac{\sigma}{2} \| \triangle\bx_i^{k+1}\|^2\\ 
          &\overset{ \eqref{H-Lip-continuity-fxy}}{\leq}&     \langle  \triangle \bx_i^{k+1} , \bg_i^{k+1} {+} \bpi_i^k{+}\sigma \triangle\blx_i^{k+1}\rangle +  ( \frac{ r_i}{2m}- \frac{\sigma}{2}  ) \| \triangle\bx_i^{k+1}\|^2\\ 
  &\overset{ \eqref{i-c-k1}}{\leq}&       \frac{4r_i}{m} \|\triangle\blx_i^{k+1}\|^2 +   ( \frac{3r_i}{4m}- \frac{\sigma}{2}  )  \| \triangle\bx_i^{k+1}\|^2\\
&\overset{ \eqref{ri-m-sigma}}{\leq}&    \frac{2\sigma}{3} \|\triangle\blx_i^{k+1}\|^2  - \frac{3\sigma}{8}    \| \triangle\bx_i^{k+1}\|^2\\
&\overset{ \eqref{ceadmm-sub3}}{=}&    \frac{2}{3\sigma} \|\triangle\bpi_i^{k+1}\|^2  -  \frac{3\sigma}{8}  \| \triangle\bx_i^{k+1}\|^2\\
&\overset{ \eqref{opt-con-xk1-3-simple}}{\leq}&             \frac{\sigma}{9}  \|\triangle\bx^{\tau_{k+1}}\|^2   -  \frac{23\sigma }{72} \| \triangle\bx_i^{k+1}\|^2\\ 
&\leq&  \frac{\sigma}{4}  \|\triangle\bx^{\tau_{k+1}}\|^2-  \frac{\sigma }{8}\|\triangle \bx_i^{k+1} \|^2,
    \end{array} 
 \end{eqnarray*} Overall, for both cases: $i\notin\ctauk$ and  $i\in\ctauk$, we have
 \begin{eqnarray} \label{gap-2}  
   \arraycolsep=1.5pt\def\arraystretch{1.5}
  \begin{array}{lcl}
e_2^k&=&\sum_{i=1}^{m} p_{i}^k = \sum_{\in\ctauk}  p_{i}^k + \sum_{i\notin\ctauk} p_{i}^k \\
 & \leq&    \sum_{i=1}^{m} ( \frac{\sigma}{4}  \|\triangle\bx^{\tau_{k+1}}\|^2-  \frac{\sigma }{8}\|\triangle \bx_i^{k+1} \|^2). 
    \end{array} 
 \end{eqnarray} 
\underline{Estimating  $e_3^k$.} Again, we have two cases. For client $i\in\ctauk $, it has the following inequalities,
\begin{eqnarray*}  
   \arraycolsep=1.4pt\def\arraystretch{1.5}
  \begin{array}{lcl}
q_i^k &:=&L(\bx^{\tau_{k+1}},\bx_i^{k+1},\bpi_i^{k+1})  - L(\bx^{\tau_{k+1}},\bx_i^{k+1},\bpi_i^{k})\\
&\overset{\eqref{Def-L}}{=}&    \langle \triangle \blx_i^{k+1}, \triangle \bpi_i^{k+1} \rangle  \overset{ \eqref{ceadmm-sub3}}{=}     \frac{1}{\sigma }\|\triangle \bpi_i^{k+1}  \|^2\\
       &\overset{(\ref{opt-con-xk1-3-simple})}{\leq}&\frac{ \sigma}{12}    \|\triangle \bx_i^{k+1}\|^2+   \frac{\sigma}{6 } \|\triangle\bx^{\tau_{k+1}}\|^2.
    \end{array} 
 \end{eqnarray*}
 For client $i\notin\ctauk $, since $\bx_i^{k+1} = \bx^{\tau_{k+1}}$ by \eqref{ceadmm-sub6}, it follows $\triangle \blx_i^{k+1}=0$ and thus  the above condition also holds. Overall
 \begin{eqnarray} \label{gap-3} 
   \arraycolsep=0pt\def\arraystretch{1.5}
  \begin{array}{lcl}
e_{3}^k =\sum_{i=1}^m q_i^k  \leq \sum_{i=1}^m  (\frac{ \sigma}{12}    \|\triangle \bx_i^{k+1}\|^2+   \frac{\sigma}{6 } \|\triangle\bx^{\tau_{k+1}}\|^2). 
    \end{array} 
 \end{eqnarray}
Combining \eqref{three-cases}, \eqref{two-cases}, \eqref{gap-2}, \eqref{gap-3}  shows the result.

ii) We focus on $s\in\K$. It follows from \eqref{Def-L} that
\begin{eqnarray}  \label{opt-con-xk1-3-s}
\arraycolsep=1pt\def\arraystretch{1.5}
\begin{array}{rll}
 \nabla_{\bx} \L(\Z) &=& - \sum_{i=1}^{m} (\bpi_i + \sigma(\bx_i-\bx)),\\
  \nabla_{\bx_i} \L(\Z) &=& \frac{1}{m}\nabla  f_i(\bx_i)  + \bpi_i + \sigma(\bx_i-\bx), \\
  \nabla_{\bpi_i} \L(\Z) &=&\bx_i-\bx.
\end{array}
\end{eqnarray}
Using the first condition in \eqref{opt-con-xk1-3-s} results in
\begin{eqnarray}
 \label{opt-con-xk1-s1}
 \arraycolsep=0pt\def\arraystretch{1.5}
\begin{array}{lll}
&&\|\nabla_{\bx} \L(\Z^{s+1})\|\\
&=&\|\sum_{i=1}^{m} (\bpi_i^{s+1} + \sigma \triangle\blx_i^{s+1}  )\|\\
&=&\|\sum_{i=1}^{m} (\bpi_i^{s} + \sigma(\bx_i^{s}-\bx^{\tau_{s+1}})+\triangle\bpi_i^{s+1}+\sigma\triangle \bx_i^{s+1})\|\\
&\overset{\eqref{opt-con-xk1-1}}{=}&\|\sum_{i=1}^{m} ( \triangle\bpi_i^{s+1}+\sigma\triangle \bx_i^{s+1})\|\\
&\leq&\sum_{i=1}^{m} (\| \triangle\bpi_i^{s+1}\|+\sigma\|\triangle \bx_i^{s+1}\|).
\end{array} \end{eqnarray} 
Using the second condition in \eqref{opt-con-xk1-3-s} results in
\begin{eqnarray}
 \label{opt-con-xk1-s2}
 \arraycolsep=1pt\def\arraystretch{1.5}
\begin{array}{lll}
&&\|\nabla_{\bx_i} \L(\Z^{s+1})\|\\
&=&\| \bg_i^{s+1} + \bpi_i^{s+1} + \sigma \triangle\blx_i^{s+1} \|\\
&\overset{\eqref{opt-con-xk1-2}}{=}&\| \bg_i^{s+1} -\overline\bg_i^{s+1}  + \sigma \triangle\blx_i^{s+1}  -\frac{1}{m}    H_i    \triangle\blx_i^{s+1} \|\\
&\leq&(\frac{r_i}{m}+\|\sigma I- \frac{1}{m}    H_i\|)\|\triangle\blx_i^{s+1}\| \\
&\overset{\eqref{ri-m-sigma}}{\leq}&\frac{7}{6}\|\sigma\triangle\blx_i^{s+1}\|  \\
&\leq&\frac{7}{6}\| \triangle\bpi_i^{s+1}\|,  
%&\overset{\eqref{opt-con-xk1-3-simple}}{\leq}& \frac{49\sigma^2}{216 } \|\triangle\bx^{\tau_{k+1}}\|^2+\frac{ 49\sigma^2}{432}    \|\triangle \bx_i^{k+1}\|^2
\end{array} \end{eqnarray} 
where the first two inequalities are from the gradient Lipschitz continuity of $f_i$ and $\sigma I\succeq \sigma I- \frac{1}{m}    H_i \succeq 0$, respectively, and the last inequality is due to \eqref{ceadmm-sub3}  if $i\in\ctauk$ and $\triangle\blx_i^{s+1}=0$ by \eqref{ceadmm-sub6} if $i\notin\ctauk$. Similarly,  using the third condition in \eqref{opt-con-xk1-3-s} results in
\begin{eqnarray}
 \label{opt-con-xk1-s3}
 \arraycolsep=0pt\def\arraystretch{1.5}
\begin{array}{lll}
 \|\nabla_{\bpi_i} \L(\Z^{s+1})\| =  \| \triangle\blx_i^{s+1}\|   \leq \frac{1}{\sigma}\| \triangle\bpi_i^{s+1}\|. 
\end{array} \end{eqnarray} 
Combining  facts \eqref{opt-con-xk1-s1}-\eqref{opt-con-xk1-s3}   
allows us to obtain
\begin{eqnarray*}
 \label{opt-con-xk1-sss}
 \arraycolsep=0pt\def\arraystretch{1.5}
\begin{array}{lll}
&&\|\nabla \L(\Z^{s+1})\| \leq \|\nabla_{\bx}  \L(\Z^{s+1})\|\\
& +&  \sum_{i=1}^{m} (\|\nabla_{\bx_i} \L(\Z^{s+1})\| + \|\nabla_{\bpi_i} \L(\Z^{s+1})\|) \\
&\leq&\sum_{i=1}^{m} ((\frac{13}{6}{+}   \frac{1}{\sigma})\| \triangle\bpi_i^{s+1}\|+\sigma\|\triangle \bx_i^{s+1}\|)\\
&\overset{\eqref{gap-pi-k}}{\leq}&\sum_{i=1}^{m} ((\frac{13}{6}{+}   \frac{1}{\sigma})(\frac{2r_i}{m}  \|\triangle\bx^{\tau_{s+1}}\|{+}\frac{r_i}{m}    \|\triangle \bx_i^{s+1}\|){+}\sigma\|\triangle \bx_i^{s+1}\|)\\
&\overset{\eqref{ri-m-sigma}}{\leq}&\sum_{i=1}^{m} ((\frac{13}{6}{+}   \frac{1}{\sigma})(\frac{\sigma}{3}  \|\triangle\bx^{\tau_{s+1}}\|{+}\frac{\sigma}{6}    \|\triangle \bx_i^{s+1}\|){+}\sigma\|\triangle \bx_i^{s+1}\|)\\
&{\leq}&\sum_{i=1}^{m}   \frac{49\sigma+12}{36} ( \|\triangle\bx^{\tau_{s+1}}\|{+} \|\triangle \bx_i^{s+1}\|)\\
&{\leq}& (2\sigma+1)  \sqrt{ m \varpi_{s+1}},
\end{array} \end{eqnarray*} 
%where $c:=\max\{\frac{13}{6}+   \frac{1}{\sigma},\sigma\}$, which  leads to
%\begin{eqnarray*}
% \label{opt-con-xk1-sss}
% \arraycolsep=0pt\def\arraystretch{1.5}
%\begin{array}{lll}
% \|\nabla \L(\Z^{s+1})\|^2  
%&\leq& 2mc^2\sum_{i=1}^{m} ( \| \triangle\bpi_i^{s+1}\|^2+\|\triangle \bx_i^{s+1}\|^2)\\
%&\leq& 2mc^2 \| \triangle\Z^{s+1}\|^2.
%\end{array} \end{eqnarray*} 
%Hence, we have $$\|\nabla \L(\Z^{s+1})\| \leq \sqrt{2m}c \| \triangle\Z^{s+1}\|,$$
showing the desired result.
\end{proof}
  \begin{lemma}\label{L-bounded-decreasing}   Let $\{\Z^{k} \}$ be the sequence generated by Algorithm \ref{algorithm-CEADMM} with $H_i=\Theta_i,i\in[m]$ and  $\sigma\geq6r/m$. The following results hold under Assumption \ref{ass-fi}.
 \begin{itemize}
 \item[i)] $\{\L(\Z^{k})\}$ is non-increasing. 
 \item[ii)] $\L(\Z^{k}) \geq f (\bx^{\tau_{k}}) \geq f^* >-\infty$ for any integer $ k\geq0$.
 \item[iii)] For any $i\in[m]$,   
\begin{eqnarray} \label{limit-5-term-0}
   \arraycolsep=0pt\def\arraystretch{1.5}
 \begin{array}{lll}
&&\underset{k \rightarrow \infty}{\lim} \triangle \bx^{\tau_{k+1}} = \underset{k \rightarrow \infty}{\lim}  \triangle \bx^{k+1}_i=\\
&&\underset{k \rightarrow \infty}{\lim}  \triangle \bpi^{k+1}_i  =\underset{k \rightarrow \infty}{\lim}  \triangle \blx^{k+1}_i =0.
    \end{array} 
 \end{eqnarray}
% \item[iv)] ${\lim}_{k \rightarrow\infty}( \bx_i^{k+1}-\bx_i^{\tau_k+1}) =0$.
 \end{itemize}
 
\end{lemma}  
\begin{proof} i) The conclusion follows from \eqref{decreasing-property-0} immediately .

 ii) From $r_iI\succeq H_i=\Theta_i\succeq 0$ and \eqref{grad-lip-theta}, we have
 \begin{eqnarray} \label{grad-lip-theta-yx}
 \arraycolsep=1.5pt\def\arraystretch{1.5}
 \begin{array}{lcl}
&&\frac{1}{m} f_i(\bx^{\tau_{k+1}})-  \frac{1}{m}f_i(\bx_i^{k+1}) \\
&\overset{\eqref{H-Lip-continuity-fxy}}{\leq} & \langle\triangle \blx^{k+1}_i  , - \blg_i^{k+1}  \rangle +\frac{r_i}{2m}\|\triangle \blx_i^{k+1}\|^2   \\
 &\overset{\eqref{opt-con-xk1-2}}{=}&       \langle \triangle \blx_i^{k+1},   \bpi_i^{k+1}  + \frac{1}{m}    H_i    \triangle\blx_i^{k+1}\rangle+\frac{r_i}{2m}\|\triangle \blx_i^{k+1}\|^2\\
  &\overset{\eqref{ri-m-sigma}}{\leq}&       \langle \triangle \blx_i^{k+1},   \bpi_i^{k+1}  \rangle+\frac{\sigma}{4}\|\triangle \blx_i^{k+1}\|^2, 
    \end{array} 
 \end{eqnarray} 
which   allows us to obtain
\begin{eqnarray*}  
  \arraycolsep=1.4pt\def\arraystretch{1.5}
 \begin{array}{lll}   
p_i^k&:=&\frac{1}{m} f_i(\bx_i^{k+1})+      \langle \triangle \blx_i^{k+1}, \bpi_i^{k+1}\rangle+  \frac{\sigma }{2}\|\triangle \blx_i^{k+1}\|^2  \\
     &\overset{\eqref{grad-lip-theta-yx}}{\geq}&    \frac{1}{m}  f_i(\bx^{\tau_{k+1}})        +\frac{\sigma  }{4} \|\triangle \blx_i^{k+1}\|^2 \geq    \frac{1}{m} f_i(\bx^{\tau_{k+1}}).   
    \end{array} 
 \end{eqnarray*}
 Using the above condition, we obtain
 \begin{eqnarray}  \label{lower-bound-L} 
  \arraycolsep=1.4pt\def\arraystretch{1.5}
 \begin{array}{lll}   
 \L(\Z^{k+1} )  &=&  \sum_{i=1}^{m} p_i^k 
 \geq   \sum_{i=1}^{m}    \frac{1}{m} f_i(\bx^{\tau_{k+1}}) \\ &=&f(\bx^{\tau_{k+1}})   \geq  f^*   \overset{\eqref{FL-opt-lower-bound}}{ >} -\infty.
    \end{array} 
 \end{eqnarray}
 iii)  From \eqref{decreasing-property-0}, we conclude that
 \begin{eqnarray*} 
    \arraycolsep=1.4pt\def\arraystretch{1.5}
  \begin{array}{lll}
 {\sum}_{k\geq0}  \frac{\sigma }{24}  \varpi_{k+1} 
&\leq&{\sum}_{k\geq0}  (  \L(\Z^{k} )  - \L(\Z^{k+1} )) \\
&=&
  \L(\Z^{0} )- {\lim}_{k\rightarrow\infty} \L(\Z^{k+1} )\overset{\eqref{lower-bound-L}}{<}+\infty. 
 \end{array}  
 \end{eqnarray*} 
The above condition means $ \triangle \bx^{\tau_{k+1}} \to0 $ and $  \triangle \bx^{k+1}_i\to 0$, yielding  $ \| \triangle \bpi_i^{k+1}  \| \rightarrow0$ by \eqref{opt-con-xk1-3} for any $i\in[m]$.  Finally, we note that $\triangle \blx^{k+1}_i=\triangle \bpi_i^{k+1}/\sigma\rightarrow0$ from \eqref{ceadmm-sub3}  if $i\in\ctauk$ and $\triangle \blx^{k+1}_i=0$ from \eqref{ceadmm-sub6}  if $i\notin\ctauk$. Overall, $\triangle \blx^{k+1}_i  \rightarrow0$, which completes the whole proof is finished.
\end{proof}

 \subsection{Proof of Theorem \ref{global-obj-convergence-exact}}
 \begin{proof} i) It follows from Lemma \ref{L-bounded-decreasing} that $\{\L(\Z^k)\}$ is non-increasing and bounded from below. Therefore,  whole sequence $\{\L(\Z^k)\}$  converges.  For $i\notin \ctauk$, we have $\blx_i^{k+1}=0$ from \eqref{ceadmm-sub6}, thereby leading to
 \begin{eqnarray*} 
   \arraycolsep=0pt\def\arraystretch{1.5}
\begin{array}{lcl}
L(\bx^{\tau_{k+1}},\bx_i^{k+1},\bpi_i^{k+1}) \overset{\eqref{Def-L}}{=} \frac{1}{m}  f_i(\bx_i^{k+1}).   
    \end{array} 
 \end{eqnarray*}  
 For $i\in \ctauk$, it follows 
\begin{eqnarray*}  
   \arraycolsep=0pt\def\arraystretch{1.5}
\begin{array}{lcl}
&& L(\bx^{\tau_{k+1}},\bx_i^{k+1},\bpi_i^{k+1}) - \frac{1}{m}  f_i(\bx_i^{k+1})\\ 
&\overset{\eqref{Def-L}}{=} &  \langle \triangle\blx_i^{k+1},    \bpi_i^{k+1}\rangle+  \frac{\sigma}{2} \| \triangle\blx_i^{k+1}\|^2  \\
      & \overset{ \eqref{ceadmm-sub3}}{=}&   \frac{1}{\sigma }  \langle \triangle \bpi_i^{k+1}, \bpi_i^{k+1}\rangle +\frac{1}{2\sigma }\|\triangle \bpi_i^{k+1} \|^2   \\
         & {=}&  \frac{1}{2\sigma } \|\bpi_i^{k+1}\|^2  - \frac{1}{2\sigma }\| \bpi_i^{k}\|^2+\frac{1}{\sigma }\|\triangle \bpi_i^{k+1} \|^2. 
    \end{array} 
 \end{eqnarray*} 
 Using the above two conditions, we can conclude that
\begin{eqnarray*} 
   \arraycolsep=1.4pt\def\arraystretch{1.5}
 \begin{array}{lll}
&& |\L(\Z^{k+1})- F(X^{k+1})|  \\ 
&=& |\sum_{i=1}^m L(\bx^{\tau_{k+1}},\bx_i^{k+1},\bpi_i^{k+1}) - \frac{1}{m}  f_i(\bx_i^{k+1})|\\  
&=& |\sum_{i\in \ctauk} \frac{1}{2\sigma } \|\bpi_i^{k+1}\|^2  - \frac{1}{2\sigma }\| \bpi_i^{k}\|^2+\frac{1}{\sigma }\|\triangle \bpi_i^{k+1} \|^2|\\ 
%&\leq& \sum_{i\in \ctauk}| \frac{1}{2\sigma } \|\bpi_i^{k+1}\|^2  - \frac{1}{2\sigma }\| \bpi_i^{k}\|^2|+\frac{1}{\sigma }\|\triangle \bpi_i^{k+1} \|^2\\
&\leq& \sum_{i=1}^m| \frac{1}{2\sigma } \|\bpi_i^{k+1}\|^2  - \frac{1}{2\sigma }\| \bpi_i^{k}\|^2|+\frac{1}{\sigma }\|\triangle \bpi_i^{k+1} \|^2 \overset{\eqref{limit-5-term-0}}{\rightarrow}  0. 
 \end{array} 
 \end{eqnarray*} 
In addition, same reasoning to show \eqref{grad-lip-theta-yx} enables to derive
\begin{eqnarray*} 
   \arraycolsep=1.4pt\def\arraystretch{1.5}
  \begin{array}{lcl}
  \frac{1}{m}f_i(\bx_i^{k+1})&-&\frac{1}{m}f_i(\bx^{\tau_{k+1}})\\
%  &\overset{\eqref{H-Lip-continuity-fxy}}{\leq}& \langle\triangle\blx_i^{k+1} , \blg^{k+1}_i\rangle + \frac{r_i}{2m}\|\triangle \blx_i^{k+1}\|^2\\
%  &\overset{\eqref{opt-con-xk1-2}}{=}&       \langle \triangle \blx_i^{k+1},   -\bpi_i^{k+1}  -\frac{1}{m}    H_i    \triangle\blx_i^{k+1}\rangle+\frac{r_i}{2m}\|\triangle \blx_i^{k+1}\|^2\\
    &\leq&       \langle \triangle \blx_i^{k+1},  - \bpi_i^{k+1}  \rangle+\frac{\sigma}{4}\|\triangle \blx_i^{k+1}\|^2, 
    \end{array} 
 \end{eqnarray*}
 which by \eqref{grad-lip-theta-yx} yields that $|q_i^k |\leq \frac{\sigma}{4}\|\triangle \blx_i^{k+1}\|^2$, where
 \begin{eqnarray*}
   \arraycolsep=1.4pt\def\arraystretch{1.5}
  \begin{array}{lcl}q_i^k &:=&\frac{1}{m}f_i(\bx_i^{k+1})-\frac{1}{m}f_i(\bx^{\tau_{k+1}})-\langle\triangle\blx_i^{k+1} , \bpi^{k+1}_i\rangle .
    \end{array} 
 \end{eqnarray*}
Therefore, the above fact  brings out
%\begin{eqnarray*} % \label{f-x-f-y-exact}
%   \arraycolsep=1.4pt\def\arraystretch{1.5}
%  \begin{array}{lcl}
%&&|L(\bx^{\tau_{k+1}},\bx_i^{k+1},\bpi_i^{k+1}) -\frac{1}{m} f_i(\bx^{\tau_{k+1}})|\\
%  &=&|q_i^k + \frac{\sigma}{2}  \|\triangle\blx_i^{k+1}\|^2| {\leq}  (\frac{\sigma}{2}+\frac{3r_i}{2m})  \|\triangle\blx_i^{k+1}\|^2 \rightarrow 0,
%   \end{array} 
% \end{eqnarray*}
% which immediately shows that
 \begin{eqnarray*} 
   \arraycolsep=1.4pt\def\arraystretch{1.5}
 \begin{array}{lll}
 |\L(\Z^{k+1})- f(\bx^{\tau_{k+1}})| 
%&\leq& \sum_{i=1}^m| L(\bx^{\tau_{k+1}},\bx_i^{k+1},\bpi_i^{k+1}) - \frac{1}{m}  f_i(\bx^{\tau_{k+1}})|\\
&=&|\sum_{i=1}^m (q_i^k + \frac{\sigma}{2}  \|\triangle\blx_i^{k+1}\|^2)|\\
& {\leq}& \sum_{i=1}^m  \frac{3\sigma}{4}  \|\triangle\blx_i^{k+1}\|^2 \rightarrow 0.
 \end{array} 
 \end{eqnarray*} 
ii) By \eqref{opt-con-xk1-1}, we derive that, for any
$\forall~k\in\K$, 
\begin{eqnarray}\label{sum-pi-0}
 \arraycolsep=1.4pt\def\arraystretch{1.5}
 \begin{array}{rcll} 
0&=&\sum_{i=1}^m  (  \bpi_i^{k} +\sigma(\bx_i^{k}-\bx^{\tau_{k+1}}))\\
&\overset{\eqref{ceadmm-sub6}}{=}&\sum_{i\in\ctauk} (  \bpi_i^{k} +\sigma  \triangle\blx_i^{k+1}-\sigma\triangle\bx_i^{k+1})\\
&+&\sum_{i\notin\ctauk} (   \bpi_i^{k+1}- \sigma\triangle\bx_i^{k+1} - \triangle \bpi_i^{k+1})\\
&\overset{\eqref{ceadmm-sub3}}{=}&
%\sum_{i\in\ctauk} (  \bpi_i^{k+1} -\sigma\triangle\bx_i^{k+1})\\
%&+&\sum_{i\notin\ctauk} (  \bpi_i^{k+1}- \sigma\triangle\bx_i^{k+1} - \triangle \bpi_i^{k+1})\\
  \sum_{i=1}^m  (  \bpi_i^{k+1}- \sigma\triangle\bx_i^{k+1} ) - \sum_{i\notin\ctauk}  \triangle \bpi_i^{k+1},
\end{array} \end{eqnarray} 
which together with  \eqref{limit-5-term-0} implies $\lim_{k(\in\K)\rightarrow\infty} \sum_{i=1}^m \bpi_i^{k+1}=0.$  Let $s:=(\tau_{k+1}-1)k_0 \in\K$. Then
\begin{eqnarray}\label{limit-pi-s}
 \arraycolsep=1.4pt\def\arraystretch{1.5}
 \begin{array}{rcll} 
\lim_{s(\in\K)\rightarrow\infty} \sum_{i=1}^m \bpi_i^{s+1}=0.
\end{array} \end{eqnarray}
Moreover, for any $k$, it  is easy to show that 
\begin{eqnarray}\label{tau-tau-s}
 \arraycolsep=1.4pt\def\arraystretch{1.5}
\begin{array}{rcll} 
s+1&=& (\tau_{k+1}-1)k_0+1 \leq k+1 \leq \tau_{k+1}k_0,\\ 
\tau_{s+1}&=&\lfloor (s+1)/k_0 \rfloor = \lfloor  \tau_{k+1}-1-1/k_0 \rfloor= \tau_{k+1}.
\end{array} \end{eqnarray}
Based on this, we now estimate $\bpi_i^{k+1}-\bpi_i^{s+1}$ for any $k$.  For any $i\in\ctauk$, we can show that  and hence
\begin{eqnarray*}
 \arraycolsep=1.4pt\def\arraystretch{1.5}
 \begin{array}{rcll} 
&&\|\bpi_i^{k+1}-\bpi_i^{s+1}\| \\
&\overset{\eqref{opt-con-xk1-2}}{=}&
\|\overline\bg_i^{k+1}  - \overline\bg_i^{s+1} +\frac{1}{m}    H_i    (\triangle\blx_i^{k+1} -   \triangle\blx_i^{s+1})\|\\
&\leq&
\frac{r_i}{m}   (\|\overline\bx^{\tau_{k+1}}  -  \overline\bx^{\tau_{s+1}}\| + \|\triangle\blx_i^{k+1}\|+\|  \triangle\blx_i^{s+1}\|)\\
&\overset{\eqref{tau-tau-s}}{=}&
\frac{r_i}{m}   (  \|\triangle\blx_i^{k+1}\|+\|  \triangle\blx_i^{s+1}\|).
\end{array} \end{eqnarray*}
For any $i\notin\ctauk$, $\bpi_i^{k+1}=\bpi_i^{s+1}=- \overline\bg_i^{s+1}$ by \eqref{ceadmm-sub7}. So, the above condition is still valid. Overall, we show that
\begin{eqnarray}\label{gap-pik-pis}
 \arraycolsep=1.4pt\def\arraystretch{1.5}
 \begin{array}{rcll} 
&&\|\bpi_i^{k+1}-\bpi_i^{s+1}\|  \leq
\frac{r_i}{m}   (  \|\triangle\blx_i^{k+1}\|+\|  \triangle\blx_i^{s+1}\|),
\end{array} \end{eqnarray}
for any $i\in[m]$ and $k\geq0$, which by \eqref{limit-5-term-0} allows us to show $\bpi_i^{k+1}-\bpi_i^{s+1}\rightarrow 0$, thereby recalling \eqref{limit-pi-s} suffices to 
\begin{eqnarray}\label{limit-pi-k}
 \arraycolsep=1.4pt\def\arraystretch{1.5}
 \begin{array}{rcll} 
\lim_{k\rightarrow\infty} \sum_{i=1}^m \bpi_i^{k+1}=0.
\end{array} \end{eqnarray}
This together with \eqref{opt-con-xk1-2} and  \eqref{limit-5-term-0} immediately gives us
 \begin{eqnarray} \label{limit-K-grad-i-exact}
\arraycolsep=0pt\def\arraystretch{1.5}
  \begin{array}{lcl}
 {\lim}_{k\rightarrow \infty} \nabla f(\bx^{\tau_{k+1}}) = {\lim}_{k\rightarrow \infty} \sum_{i=1}^{m}  \blg_i^{k+1} = 0. 
   \end{array}
  \end{eqnarray} 
Finally, the above condition together with $ \triangle \blx_i^{k+1} =(\bx_i^{k+1}-\bx^{\tau_{k+1}})\rightarrow0$ and the gradient Lipschitz continuity yields that
 \begin{eqnarray*} %\label{limit-K-grad-exact}
 \arraycolsep=8pt\def\arraystretch{1.5}
  \begin{array}{llllllll}
 {\lim}_{k \rightarrow \infty} \sum_{i=1}^{m}  w_i \nabla f_i(\bx^{\tau_{k+1}}) =0,
   \end{array}
  \end{eqnarray*} 
 completing the whole proof. \end{proof}
 
 \subsection{Proof of Theorem \ref{global-convergence-exact} }  
 \begin{proof} 
i) It follows from Lemma \ref{L-bounded-decreasing}  i) and \eqref{lower-bound-L} that
\begin{eqnarray} 
   \arraycolsep=1.4pt\def\arraystretch{1.5}
  \begin{array}{lll}
\L(\Z^0)  \geq  \L(\Z^{k+1})  \geq   \sum_{i=1}^{m}    w_i  f_i(\bx^{\tau_{k+1}}) = f(\bx^{\tau_{k+1}}), 
    \end{array} 
 \end{eqnarray}
which  implies $\bx^{\tau_{k+1}}\in\S$ and hence  $\{\bx^{\tau_{k+1}}\}$ is bounded due to  the boundedness of $\S$.  This calls forth the boundedness of $\{\bx_i^{k+1}\}$ as $\triangle\blx_i^{k+1}\rightarrow0$ from \eqref{limit-5-term-0}. Then the boundedness of sequence $\{\bpi_i^{k+1}\}$ can be ensured because of
  \begin{eqnarray*} 
\arraycolsep=1.4pt\def\arraystretch{1.5}
 \begin{array}{lcl}
 \|\bpi_i^{k+1}\|  &\overset{\eqref{opt-con-xk1-2}}{=}&
\|\overline\bg_i^{k+1}  +\frac{1}{m}    H_i   \triangle\blx_i^{k+1} \|  \\  &\leq&  \|\blg_i^{k+1} - \bg_i^{0}\|+ \|\bg_i^{0}\| +\frac{r_i}{m}    \| \triangle\blx_i^{k+1} \|   \\ 
  &\overset{\eqref{Lip-r}}{\leq}&\frac{r_i}{m}    \|  \bx^{\tau_{k+1}}-\bx^{0}_i\|+ \|\bg_i^{0}\| +\frac{r_i}{m}    \| \triangle\blx_i^{k+1} \|<+\infty,  
     \end{array} 
 \end{eqnarray*}
where `$<$' is from the boundedness of $\{\bx^{\tau_{k+1}}\}$. Overall, $\{\Z^{k+1}\}$ is bounded. Let $\Z^{\infty}$ be any accumulating point of the sequence, 
 it follows from \eqref{opt-con-xk1-2}  and $ \triangle \blx_i^{k+1} \rightarrow0$ that 
 \begin{eqnarray*}
 \arraycolsep=1.4pt\def\arraystretch{1.5}
  \begin{array}{lll}
0&=& \overline\bg_i^{k+1}  +  \bpi_i^{k+1} + \frac{1}{m}    H_i    \triangle\blx_i^{k+1}\\
&=& \bg_i^{k+1}  +  \bpi_i^{k+1} + \overline\bg_i^{k+1}- \bg_i^{k+1}+ \frac{1}{m}    H_i    \triangle\blx_i^{k+1}\\
&\rightarrow& \frac{1}{m} \nabla f_i(\bx^{\infty}_i) + \bpi_i^{\infty}.
   \end{array}
  \end{eqnarray*} 
Moreover,   \eqref{limit-pi-k} and \eqref{limit-5-term-0}  suffice to  $\sum_{i=1}^{m} \bpi_i^{\infty}=0$ and $ \bx^{\infty}_i- \bx^{\infty}=0$. 
%   \begin{eqnarray*}
% \arraycolsep=1.4pt\def\arraystretch{1.5}
%  \begin{array}{l}
%  \sum_{i=1}^{m} \bpi_i^{\infty}=0,\qquad \bx^{\infty}_i- \bx^{\infty}=0.
%   \end{array}
%  \end{eqnarray*}
 By recalling \eqref{opt-con-FL-opt-ver1},  $\Z^{\infty}$  is a stationary point of (\ref{FL-opt-ver1}) and $\bx^{\infty}$ is a stationary point of   (\ref{FL-opt}). 
  
ii) We first prove the result if $\bx^{\infty}$ is isolated.  Since $\triangle \bx^{\tau_{k+1}} \rightarrow0$ and $\bx^{\infty}$ being isolated,  whole sequence $\{\bx^{\tau_{k+1}}\}$ converges to $\bx^{\infty}$ by \cite[Lemma 4.10]{more1983computing}. This together with $ \triangle \blx_i^{k+1}\rightarrow0$ and \eqref{opt-con-xk1-2} implies that $\{X^{k}\}$ and $\{\Pi^k\}$ converge to $X^{\infty}$ and $\Pi^\infty$.

{Now we show the result if every $f_i$ has KL property. For any $t\in{\mathbb N}:=\{0,1,2,\cdots\}$, denote  
   \begin{eqnarray*}
 \arraycolsep=1.4pt\def\arraystretch{1.5}
  \begin{array}{l}
 \F^t =: \L(\Z^{tk_0+1}),~ \nabla \F^t := \nabla\L(\Z^{tk_0+1}). 
   \end{array}
  \end{eqnarray*} 
Let $\Omega$ be the set of all accumulating points of sequence   $\{\Z^{tk_0+1}:t\in{\mathbb N}\}$. Then  $\L(\Z ),\forall \Z\in\Omega$ have the same value since  whole sequence $\{\L(\Z^{tk_0+1})\}$  converges. Denote
   \begin{eqnarray*}
 \arraycolsep=1.4pt\def\arraystretch{1.5}
  \begin{array}{l}
\F^\infty := \L(\Z ),~ \forall \Z\in\Omega.
   \end{array}
  \end{eqnarray*} 
By the non-increasing property of $\{\L(\Z^k)\}$, we have
  \begin{eqnarray} 
 \label{decreasing-property-0-tk0} 
  \arraycolsep=1.4pt\def\arraystretch{1.5}  
\begin{array}{lcl}
\F^{t+1} &=&  \L(\Z^{(t+1)k_0+1})\\
& \overset{ \eqref{decreasing-property-0}}{\leq}&  {\L} (\Z^{(t+1)k_0}) - \frac{\sigma }{24}\varpi_{(t+1)k_0+1}\\
&  {\leq}&  {\L} (\Z^{tk_0+1}) - \frac{\sigma }{24}\varpi_{(t+1)k_0+1}\\
& =& \F^{t}  - \frac{\sigma }{24}\varpi_{(t+1)k_0+1}\\
    \end{array}  
 \end{eqnarray} 
 and 
 \begin{eqnarray} 
 \label{grad-decreasing-property-0-tk0}   
\begin{array}{lll}
\|\nabla\F^t\|= \|\nabla \L(\Z^{tk_0+1})\|\overset{\eqref{grad-decreasing-property-0}  }{\leq}  (2\sigma+1)  \sqrt{ m  \varpi_{tk_0+1}}.
    \end{array}  
 \end{eqnarray}  
Now for any $\eta>0$  and $\delta>0$, there exists $t_0\in\N$ such that
 \begin{eqnarray} \label{exist-t0}
   \arraycolsep=1.0pt\def\arraystretch{1.5}  
% \label{grad-decreasing-property-0-tk0}   
\begin{array}{lll}
\Z^{t k_0+1}&\in& \{\Z: {\rm dist}(\Z,\Omega)\leq\delta\} \cap\\
&& \{\Z:  \F^\infty\leq \L(\Z)\leq \F^\infty + \eta\},~~t\geq t_0, 
    \end{array}  
 \end{eqnarray} 
due to $\lim_{t\to\infty}\Z^{t k_0+1}\in\Omega$ and $\F^t\to\F^\infty$. Since every $f_i$ has KL property, so has $\L$. Then from the KL property, there exists a desingularizing function $\varphi$ such that 
  \begin{eqnarray*} 
   \arraycolsep=1.0pt\def\arraystretch{1.5}  
\begin{array}{lll}
 \varphi'(\F^t- \F^\infty) \|\nabla\F^t\| \geq 1,
    \end{array}  
 \end{eqnarray*} 
 namely,
   \begin{eqnarray} 
   \label{desingularizing}  
   \arraycolsep=1.0pt\def\arraystretch{1.5}     
\begin{array}{lll}
 \varphi'(\F^t- \F^\infty) \geq \frac{1}{\|\nabla\F^t\|} \overset{\eqref{grad-decreasing-property-0-tk0} }{\geq} \frac{1}{ (2\sigma+1)  \sqrt{  m \varpi_{tk_0+1}}}.
    \end{array}  
 \end{eqnarray} 
The above condition and $ \varphi$ being concave bring out
    \begin{eqnarray} 
   \arraycolsep=1.0pt\def\arraystretch{1.5}  
 \label{desingularizing-2}   
\begin{array}{lcl}
&&\varphi(\F^{t+1}- \F^\infty)-\varphi(\F^t- \F^\infty)\\
 &\leq&  \varphi'(\F^t- \F^\infty)  (\F^{t+1}- \F^t)
  \overset{(\ref{decreasing-property-0-tk0})  }{\leq} \frac{-    \varpi_{(t+1)k_0+1} }{ a \sqrt{ \varpi_{tk_0+1}}},
    \end{array}  
 \end{eqnarray} 
 where $a:= {24(2\sigma+1)  \sqrt{  m}}/{ \sigma}$, which suffices to
     \begin{eqnarray*} 
   \arraycolsep=1.0pt\def\arraystretch{1.5}  
 \label{desingularizing-2}   
\begin{array}{lcl}
&&\sqrt{ \varpi_{(t+1)k_0+1}} \\ &\leq & 
 \sqrt{ \sqrt{\varpi_{tk_0+1}}a (\varphi(\F^t- \F^\infty)-\varphi(\F^{t+1}- \F^\infty)) } \\
 &\leq & \frac{1}{2}  \sqrt{\varpi_{tk_0+1}} + \frac{a}{2}(\varphi(\F^t- \F^\infty)-\varphi(\F^{t+1}- \F^\infty)).
    \end{array}  
 \end{eqnarray*} 
 Summing the both sides of the above condition yields 
      \begin{eqnarray} 
   \arraycolsep=1.0pt\def\arraystretch{1.5}  
 \label{desingularizing-3}   
\begin{array}{lcl}
 \sum_{t\geq0} \sqrt{ \varpi_{(t+1)k_0+1}}  \leq  \sqrt{ \varpi_{1}}  +a \varphi(\F^0- \F^\infty)<\infty.
    \end{array}  
 \end{eqnarray} 
 We note from \eqref{def-Y-k1} that
 \begin{eqnarray*} 
  \arraycolsep=1.4pt\def\arraystretch{1.5}  
\begin{array}{lll}
\varpi_{k+1}{=} \sum_{i=1}^m ( \| \triangle\bx^{\tau_{k+1}}\|^2  {+}  \| \triangle\bx_i^{k+1}\|^2){=}\|\Y^{k+1}-\Y^{k}\|^2, 
    \end{array}  
 \end{eqnarray*} 
 where $\Y^{k}{:=}(\bx^{\tau_k},\cdots,\bx^{\tau_k},\bx_1^k,\cdots,\bx_m^k)$. This and \eqref{desingularizing-3} mean that $\{\Y^{tk_0+1}\}$ is a Cauchy sequence and hence is convergent, resulting in the convergence of sequences $\{\bx^{\tau_{tk_0}}:t\in\N\}$ and $\{X^{tk_0+1}:t\in\N\}$. These by \eqref{opt-con-xk1-2} and $\triangle\blx_i^{k+1}\to 0$ indicate $\{\Pi^{tk_0+1}\}$ also converges. Overall, $\{\Z^{tk_0+1}\}$ converges. Recalling \eqref{limit-5-term-0} and $k_0$ is a finite number, we can conclude $\{\Z^{k}\}$ has the same convergence behaviour of $\{\Z^{tk_0+1}\}$.} 
  \end{proof}

 \subsection{Proof of Corollary \ref{L-global-convergence}}
 
\begin{proof} i) The convexity of $f$ and the optimality of $\bx^*$ lead to
    \begin{eqnarray}  \label{convexity-optimality}
   \arraycolsep=1.4pt\def\arraystretch{1.5}
   \begin{array}{lll}
f(\bx^{\tau_{k}}) \geq  f(\bx^{*}) \geq  f(\bx^{\tau_{k}}) + \langle \nabla f(\bx^{\tau_{k}}), \bx^{*}-\bx^{\tau_{k}} \rangle. 
    \end{array} 
 \end{eqnarray} 
Theorem \ref{global-obj-convergence-exact} ii) states that   
   \begin{eqnarray*}  
   \arraycolsep=1.4pt\def\arraystretch{1.5}
   \begin{array}{lll}
 {\lim}_{k \rightarrow \infty} \nabla F(X^{k})  ={\lim}_{k \rightarrow\infty} \nabla f(\bx^{\tau_{k}}) =0.
    \end{array} 
 \end{eqnarray*} 
 Using this and the boundedness of $\{\bx^{\tau_{k}}\}$ from Theorem \ref{global-convergence-exact}, we take the limit of both sides of \eqref{convexity-optimality} to derive that  $ f(\bx^{\tau_{k}})\rightarrow f(\bx^{*})$, which recalling Theorem  \ref{global-obj-convergence-exact} i) yields \eqref{L-global-convergence-limit}. 
 
 ii) The conclusion follows from Theorem  \ref{global-convergence-exact} ii) and the fact that  the stationary points are equivalent to optimal solutions if $f$ is convex.
 
 iii)   The strong convexity of $f$ means that
 there is a positive constance $\nu$ such that 
 \begin{eqnarray}  \label{strong-convexity-nu}
   \arraycolsep=1.4pt\def\arraystretch{1.5}
  \begin{array}{llll}
  f( \bx^{\tau_{k}}) -f( \bx^*)
 &\geq& \langle \nabla f( \bx^*), \bx^{\tau_{k}}-\bx^*\rangle + \frac{\nu}{2}\|\bx^{\tau_{k}}-\bx^*\|^2\\
 & =&  \frac{\nu}{2}\|\bx^{\tau_{k}}-\bx^*\|^2, 
    \end{array} 
 \end{eqnarray}
where the equality is due to \eqref{opt-con-FL-opt-ver1}. Taking limit of both sides of the above inequality immediately shows $ \bx^{\tau_{k}}\rightarrow\bx^*$ since $ f( \bx^{\tau_{k}}) \rightarrow f( \bx^*)$. This together with \eqref{limit-5-term-0} yields $ \bx_i^k\rightarrow\bx^*$. Finally, $ \bpi_i^k\rightarrow\bpi_i^*$ because of   
 \begin{eqnarray*}  
   \arraycolsep=1.4pt\def\arraystretch{1.5}
  \begin{array}{lcl}
  \|\bpi_i^{k}-\bpi_i^{*}\|&\overset{\eqref{opt-con-xk1-2},\eqref{opt-con-FL-opt-ver1}}{=}&\|\overline\bg_i^{k} + \frac{1}{m}    H_i    \triangle\blx_i^{k}-\frac{1}{m}  \nabla f_i(\bx^*)\|\\
  & \overset{\eqref{Lip-r}}{\leq}&   \frac{r_i}{m}(\| \bx^{\tau_k}-\bx^*\|+\|\triangle\blx_i^{k}\|) \rightarrow 0,
      \end{array} 
 \end{eqnarray*}
 displaying the desired result.
\end{proof}
 \subsection{Proof of Theorem \ref{complexity-thorem-gradient}}
%\subsection{Key lemma }
%\begin{lemma}\label{grad-L-bounded}  Let $\{(\bx^{\tau_{k}},X^{k},\Pi^{k})\}$ be the sequence generated by Algorithm \ref{algorithm-CEADMM} with $H_i=\Theta_i,i\in[m]$ and  $\sigma >6r/m$. If Assumption \ref{ass-fi} holds, then   for any $k\in\K$,
%\begin{eqnarray}  \label{grad-L-bounded-eq}
%   \arraycolsep=0pt\def\arraystretch{1.5}
%  \begin{array}{lllll}
%&& \|\nabla  f(\bx^{\tau_{k+1}})\|^2\\
%& \leq&  {5m\sigma ^2} \sum_{i=1}^m    (\| \triangle \bx^{k+1}_i\|^2 + \|\triangle\bx^{\tau_{k+1}}\|^2). 
%    \end{array}  
% \end{eqnarray} 
%\end{lemma}  
\begin{proof} For any $j\geq1$ and \eqref{opt-con-xk1-3-simple},  there is
  \begin{eqnarray}\label{Lip-dpi-dxi} 
   \arraycolsep=1.4pt\def\arraystretch{1.5}
  \begin{array}{lllll}\| \triangle \bpi_i^{j+1} \|^2   
  & \leq& \frac{\sigma ^2}{6}(\|\triangle \bx_i^{j+1}\|^2+     \|\triangle\bx^{\tau_{j+1}}\|^2).
\end{array}   \end{eqnarray}
We note that $\triangle \blx^{k+1}_i=\triangle \bpi_i^{k+1}/\sigma\rightarrow0$ from \eqref{ceadmm-sub3}  if $i\in\ctauk$ and $\triangle \blx^{k+1}_i=0$ from \eqref{ceadmm-sub6}  if $i\notin\ctauk$. Therefore,
\begin{eqnarray} \label{gradient-of-xk-xtuak-0}
   \arraycolsep=1.4pt\def\arraystretch{1.5}
  \begin{array}{lcl}
  \| \triangle \blx^{k+1}_i \| 
  \leq \|\triangle \bpi_i^{k+1}/\sigma \|, ~~\forall~ i\in[m].   
    \end{array}  
 \end{eqnarray}
Now we focus on  $s\in\K$. This by \eqref{sum-pi-0} results in 
\begin{eqnarray*}
 \arraycolsep=1.4pt\def\arraystretch{1.5}
 \begin{array}{rcll} 
\sum_{i=1}^m  \bpi_i^{s+1} = \sum_{i=1}^m  \sigma\triangle\bx_i^{s+1} + \sum_{i\notin\mathcal C^{\tau_{s+1}}}  \triangle \bpi_i^{s+1},
\end{array} \end{eqnarray*} 
which leads to
\begin{eqnarray*}
 \arraycolsep=1.4pt\def\arraystretch{1.5}
 \begin{array}{rcll} 
\|\sum_{i=1}^m  \bpi_i^{s+1}\|^2
\leq m\sum_{i=1}^m 2(\|  \sigma \triangle\bx_i^{s+1}\|^2+\| \triangle \bpi_i^{s+1} \|^2).
%&\leq&m\sum_{i=1}^m \frac{7\sigma ^2}{3}(\|\triangle \bx_i^{s+1}\|^2+     \|\triangle\bx^{\tau_{s+1}}\|^2).
\end{array} \end{eqnarray*}
Using this condition generates  
  \begin{eqnarray}\label{fact-max}
  \arraycolsep=1.5pt\def\arraystretch{1.5}
  \begin{array}{lcl}
&&\|\nabla f(\bx^{\tau_{s+1}})\|^2 =\|\sum_{i=1}^m  \blg_i^{{s+1}}\|^2\\
&\overset{\eqref{opt-con-xk1-2}}{=}& \|\sum_{i=1}^m  (\bpi_i^{s+1} + \frac{1}{m}    H_i    \triangle\blx_i^{s+1} ) \|^2\\
&{\leq}& 2\|\sum_{i=1}^m  \bpi_i^{s+1}\|^2 + 2m\sum_{i=1}^m\frac{r_i^2}{m^2}   \| \triangle\blx_i^{s+1}  \|^2\\
&\overset{\eqref{ri-m-sigma},\eqref{gradient-of-xk-xtuak-0}}{\leq}& 2\|\sum_{i=1}^m  \bpi_i^{s+1}\|^2 + \frac{m}{18}\sum_{i=1}^m    \| \triangle\bpi_i^{s+1}  \|^2\\
&\leq&  m\sum_{i=1}^m (4\|  \sigma \triangle\bx_i^{s+1}\|^2+5\| \triangle \bpi_i^{s+1} \|^2)\\
&\overset{\eqref{Lip-dpi-dxi} }{\leq}&  5m\sigma ^2 \sum_{i=1}^m   (\|\triangle \bx_i^{s+1}\|^2+     \|\triangle\bx^{\tau_{s+1}}\|^2)\\
&\overset{\eqref{decreasing-property-0} }{\leq}&100m\sigma  (\L (\Z^s)-\L(\Z^{s+1})).
   \end{array}
  \end{eqnarray}
% finishing the proof.
%\end{proof}
% \subsection{Proof of Theorem \ref{complexity-thorem-gradient} }
% \begin{proof}
%Let $s\in\K$, then the above condition indicates
% \begin{eqnarray} \label{fact-max}
%   \arraycolsep=0pt\def\arraystretch{1.5}
%  \begin{array}{lcl}
% \|\nabla  f(\bx^{\tau_{s+1}})\|^2 
%&\overset{\eqref{grad-L-bounded-eq}}{\leq}& {5m\sigma ^2} \sum_{i=1}^m   (\| \triangle \bx^{s+1}_i\|^2 + \|\triangle\bx^{\tau_{s+1}}\|^2)\\ 
%&\overset{\eqref{decreasing-property-0}}{\leq}& \frac{5m\sigma ^2}{\eta} (\L ^s-\L ^{s+1}).
%    \end{array}  
% \end{eqnarray} 
Since sequence $\{\L(\Z^{k})\}$ is non-increasing from Lemma \ref{L-bounded-decreasing}, it has  $\L(\Z^{tk_0+1}) \geq \L (\Z^{(t+1)k_0}) \geq f^*$ by Lemma \ref{L-bounded-decreasing} for any $t\geq0$, thereby resulting in
   \begin{eqnarray} \label{fact-sum-0K}
   \arraycolsep=1.4pt\def\arraystretch{1.5}
  \begin{array}{lllll}
 &&\sum_{t=0}^{\tau_{k+1}-1}( \L(\Z^{tk_0})-\L(\Z^{tk_0+1}))\\
&=& \L(\Z^{0})-\sum_{t=0}^{\tau_{k+1}-2}(\L(\Z^{tk_0  +1})-\L (\Z^{(t+1)k_0}) )\\
&-&\L(\Z^{(\tau_{k+1}-1) k_0  +1}) \\
&\leq& \L(\Z^{0})- \L(\Z^{(\tau_{k+1}-1) k_0  +1}) \leq \L(\Z^{0})- f ^*.  
    \end{array}  
 \end{eqnarray}
 We note that for $j=0,1,2,\ldots,\tau_{k+1} k_0-1$,
  \begin{eqnarray*}  
   \arraycolsep=1.5pt\def\arraystretch{1.0}
\tau_{j+1}=\left\{  \begin{array}{lll}
 1,& j=0,1,\ldots,k_0-1,\\
 2,& j=k_0,k_0+1,\ldots,2k_0-1,\\
 ~ \vdots & ~~~\vdots\\
 \tau_{k+1},& j=(\tau_{k+1}-1) k_0,\ldots,k,\ldots,\tau_{k+1}k_0-1.
    \end{array}  \right.
 \end{eqnarray*} 
 Using the  above three facts and $k<\tau_{k+1}k_0-1$, we derive
   \begin{eqnarray*} 
   \arraycolsep=1.5pt\def\arraystretch{1.5}
  \begin{array}{lcl}
&&\min_{j\in[k]}   \|\nabla  f(\bx^{\tau_{j}})\|^2\\
 &\leq&\frac{1}{k}\sum_{j=0}^{k-1}\|\nabla  f(\bx^{\tau_{j+1}})\|^2\\
 &\leq&\frac{1}{k} \sum_{j=0}^{\tau_{k+1}k_0-1}\|\nabla  f(\bx^{\tau_{j+1}})\|^2\\
 &=& \frac{k_0}{k} \sum_{t=0}^{\tau_{k+1}-1}   \|\nabla  f(\bx^{\tau_{tk_0+1}})\|^2\\
&\overset{\eqref{fact-max}}{\leq}&    \frac{100m\sigma k_0}{k} \sum_{t=0}^{\tau_{k+1}-1 }\left( \L(\Z^{tk_0})-\L(\Z^{tk_0+1})\right)\\
&\overset{\eqref{fact-sum-0K}}{\leq}& \frac{100m\sigma k_0}{k}   (\L(\Z^0)-f^*), 
    \end{array}  
 \end{eqnarray*}
completing the proof.
 \end{proof}

  \subsection{Proof of Theorem \ref{complexity-thorem-F-F}}
 \begin{proof}  Again, for any $t\in{\mathbb N}:=\{0,1,2,\cdots\}$, denote  
   \begin{eqnarray*}
 \arraycolsep=1.4pt\def\arraystretch{1.5}
  \begin{array}{l}
 \F^t =: \L(\Z^{tk_0+1}),~ \triangle \F^t:=\F^t- \F^\infty. 
   \end{array}
  \end{eqnarray*} 
Throughout the proof, assume $\F^t \neq \F^\infty$.  Recalling  $\varphi(z)=\frac{\sqrt{c}}{1-\theta} z^{1-\theta}$, we obtain
    \begin{eqnarray*} 
   \arraycolsep=1.0pt\def\arraystretch{1.5}   
\begin{array}{lcl}
1&\overset{\eqref{desingularizing}}{\leq}&  (\varphi'(\triangle \F^t))^2 (2\sigma+1) ^2 m  \varpi_{tk_0+1}\\
&=& c(\triangle\F^t)^{-2\theta} (2\sigma+1) ^2 m  \varpi_{tk_0+1}\\
&\overset{\eqref{decreasing-property-0-tk0} }{\leq}& c(\triangle\F^t)^{-2\theta}   \frac{24 m(2\sigma+1) ^2}{\sigma}(\F^{t-1}- \F^t).
    \end{array}  
 \end{eqnarray*} 
  By letting $\rho: = \frac{24 m c(2\sigma+1) ^2}{\sigma}$, we have
      \begin{eqnarray} \label{rho-F-F-F}
   \arraycolsep=1.0pt\def\arraystretch{1.5}   
\begin{array}{lcl}
 \rho(\triangle \F^{t-1} - \triangle\F^t) \geq  (\triangle \F^t)^{2\theta}.
    \end{array}  
 \end{eqnarray} 
 Now we prove the results by three cases:
 \begin{itemize}[leftmargin=10pt]
 \item If $\theta=0$, then $\rho(\triangle \F^{t-1} - \triangle\F^t)\geq1$. However, $\triangle \F^{t} \to 0$, which leads to a contradiction. Therefore, we have that $\F^t = \F^\infty$ when $t$ is over a threshold $t_1\geq t_0>0$.
 \item If $\theta\in(0,1/2]$, then there is a $t_2\geq t_0>0$ such as $\triangle \F^t\in[0,1]$ for any $t\geq  t_2 $ as $\triangle \F^{t} \to 0$, which by \eqref{rho-F-F-F} yields
       \begin{eqnarray*} 
   \arraycolsep=1.0pt\def\arraystretch{1.5}   
\begin{array}{lcl}
 \rho(\triangle \F^{t-1} - \triangle\F^t) \geq  (\triangle \F^t)^{2\theta} \geq \triangle \F^t.
    \end{array}  
 \end{eqnarray*} 
This allows us to derive that
        \begin{eqnarray*} 
   \arraycolsep=0pt\def\arraystretch{1.5}   
\begin{array}{lcl}
  \triangle \F^t {\leq}  \frac{\rho}{\rho+1}\triangle \F^{t-1} 
    {\leq}  (\frac{\rho}{\rho+1})^2\triangle \F^{t-2} 
    \cdots {\leq} (\frac{\rho}{\rho+1})^{t-t_2}\triangle \F^{t_2}.
    \end{array}  
 \end{eqnarray*}  
 \item If $\theta\in(1/2,1)$, then $\phi(z):= \frac{\rho}{1-2\theta} z^{1-2\theta}$ is an increasing function. If $(\triangle \F^{t-1})^{-2\theta}\geq (\triangle \F^{t})^{-2\theta}/2$, then 
         \begin{eqnarray*} 
   \arraycolsep=1pt\def\arraystretch{1.75}   
\begin{array}{lcl}
&& \phi(\triangle \F^{t-1}) - \phi(\triangle \F^t)\\
 &=& \int^{\triangle \F^{t-1}}_{\triangle \F^{t}} \phi'(z) dz = \int^{\triangle \F^{t-1}}_{\triangle \F^{t}} \rho z^{-2\theta} dz \\
  &\geq&  \rho(\triangle \F^{t-1}-\triangle \F^{t})  (\triangle \F^{t-1})^{-2\theta} \\ 
  &\geq&  \frac{1}{2}\rho(\triangle \F^{t-1}-\triangle \F^{t})  (\triangle \F^{t})^{-2\theta} \overset{\eqref{rho-F-F-F} }{\geq}  \frac{1}{2}.
    \end{array}  
 \end{eqnarray*}  
 If $(\triangle \F^{t-1})^{-2\theta}\leq (\triangle \F^{t})^{-2\theta}/2$, then 
          \begin{eqnarray*} 
   \arraycolsep=1pt\def\arraystretch{1.5}   
\begin{array}{lcl}
&& (\triangle \F^{t-1})^{-2\theta}\leq (\triangle \F^{t})^{-2\theta}/2\\
%&\Longleftrightarrow& 2(\triangle \F^{t})^{2\theta}  \leq (\triangle \F^{t-1})^{2\theta}\\
&\Longleftrightarrow& 2^{\frac{1}{2\theta}}(\triangle \F^{t})  \leq (\triangle \F^{t-1}) \\ 
&\Longleftrightarrow& 2^{\frac{1-2\theta}{2\theta}}(\triangle \F^{t})^{1-2\theta}  \geq (\triangle \F^{t-1})^{1-2\theta}, 
    \end{array}  
 \end{eqnarray*}
 which by the non-increasing property of $\{\triangle\F^{t}\}$ suffices to 
          \begin{eqnarray*} 
   \arraycolsep=1pt\def\arraystretch{1.5}   
\begin{array}{lcl}
 &&(1-2\theta) (\phi(\triangle \F^{t-1}) - \phi(\triangle \F^t)) \\
 &=& \rho  \left((\triangle \F^{t-1})^{1-2\theta}-(\triangle \F^{t})^{1-2\theta}\right)\\
 &\leq& \rho(2^{\frac{1-2\theta}{2\theta}}-1 ) (\triangle \F^{t})^{1-2\theta}\\
  &\leq& \rho(2^{\frac{1-2\theta}{2\theta}}-1 ) (\triangle \F^{0})^{1-2\theta}=:c.
    \end{array}  
 \end{eqnarray*}
 Let $a{:=}\min\{\frac{1}{2},\frac{c}{1-2\theta}\}$. Both cases lead to
           \begin{eqnarray*}  
\begin{array}{lcl}
 \phi(\triangle \F^{t-1}) - \phi(\triangle \F^t) \geq  a.
    \end{array}  
 \end{eqnarray*}
 We note from \eqref{exist-t0}, the KL property holds when $t\geq t_0$. Therefore, the above inequality holds for $t\geq t_3=t_0$. Summing the both sides of the above inequality for $j=t_3,t_3+1,\cdots,t$ derives that
            \begin{eqnarray*} 
   \arraycolsep=1pt\def\arraystretch{1.5}   
\begin{array}{lcl}
 \phi(\triangle \F^{t_3-1}) - \phi(\triangle \F^t) &\geq&\sum_{j=t_3}^t a = (t-t_3)a.
    \end{array}  
 \end{eqnarray*}
 This condition indicates
             \begin{eqnarray*} 
   \arraycolsep=1pt\def\arraystretch{1.75}   
\begin{array}{lcl}
\frac{\rho}{1-2\theta} (\triangle \F^t)^{1-2\theta}& =&  \phi(\triangle \F^t)\\
& \leq & -(t-t_3)a +  \frac{\rho}{1-2\theta} (\triangle \F^{t_3-1})^{1-2\theta}\\
& \leq & -(t-t_3)a.
    \end{array}  
 \end{eqnarray*}
 Henceforth, we obtain
              \begin{eqnarray*} 
   \arraycolsep=1pt\def\arraystretch{1.75}   
\begin{array}{lcl}
\triangle \F^t  
 \leq  (\frac{(2\theta-1)a}{\rho}(t-t_3))^{\frac{1}{1-2\theta}}.
    \end{array}  
 \end{eqnarray*} 
  \end{itemize} 
  Finally, the above three cases,   $f(\bx^{\tau_k}){\leq} \L(\Z^k)$ for any $k{\geq}0$ from lemma \ref{L-bounded-decreasing} ii), and \eqref{L-local-convergence-limit-1}  can derive the conclusion.
  \end{proof}

\appendices

%\section*{Acknowledgment}
%The authors sincerely thank the associate editor and the five referees for their constructive comments, which have significantly improved the quality of the paper.

\ifCLASSOPTIONcaptionsoff
  \newpage
\fi

\end{document}